\documentclass{article} 
\usepackage{iclr2025_conference,times}

\usepackage{amsmath,amsfonts,bm}

\def\eqref#1{equation~\ref{#1}}

\def\1{\bm{1}}

\DeclareMathAlphabet{\mathsfit}{\encodingdefault}{\sfdefault}{m}{sl}
\SetMathAlphabet{\mathsfit}{bold}{\encodingdefault}{\sfdefault}{bx}{n}

\newcommand{\R}{\mathbb{R}}

\usepackage{graphicx} 
\usepackage{subcaption}
\usepackage{amsmath}
\usepackage{amssymb}
\usepackage{mathtools}
\usepackage{amsthm}
\usepackage{amsfonts,mathrsfs}
\usepackage{pgfplots}
\usepackage{multirow}
\usepackage{pgfplots}
\usepackage{pgfplotstable}

\definecolor{color1}{HTML}{1f77b4}
\definecolor{color2}{HTML}{ff7f0e}
\definecolor{color3}{HTML}{2ca02c}
\definecolor{color4}{HTML}{e4bcad}
\definecolor{color5}{HTML}{76c8c8}
\definecolor{color6}{HTML}{c80064}
\definecolor{color7}{HTML}{3c4e4b}
\definecolor{color8}{HTML}{7f7f7f}
\definecolor{color9}{HTML}{bcbd22}
\definecolor{color10}{HTML}{17becf}
\definecolor{color11}{HTML}{aec7e8}
\definecolor{color12}{HTML}{ffbb78}
\definecolor{color13}{HTML}{98df8a}
\definecolor{color14}{HTML}{ff9896}
\definecolor{color15}{HTML}{c5b0d5}
\definecolor{color16}{HTML}{c49c94}
\definecolor{color17}{HTML}{f7b6d2}
\definecolor{color18}{HTML}{c7c7c7}
\definecolor{color19}{HTML}{dbdb8d}
\definecolor{color20}{HTML}{9edae5}
\definecolor{color21}{HTML}{ad494a}
\pgfplotsset{compat=1.5}
\usepackage{tabularx}
\usepackage{booktabs}
\usepackage{hyperref}
\usepackage{url}

\theoremstyle{plain}
\newtheorem{theorem}{Theorem}[section]

\newtheorem{lemma}[theorem]{Lemma}

\theoremstyle{definition}

\theoremstyle{remark}

\newcommand{\blue}[1]{\textcolor{black}{#1}}

\newcommand{\method}{RandLoRA} 

\title{\method{}: Full-rank parameter-efficient fine-tuning of large models}

\author{%
  \quad Paul Albert 
  \quad Frederic Z. Zhang
  \quad Hemanth Saratchandran \\
  \quad \textbf{Cristian Rodriguez-Opazo}  
  \quad \textbf{Anton van den Hengel}
   \quad \textbf{Ehsan Abbasnejad} \vspace{5pt} \\
   \quad Australian Institute for Machine Learning \quad The University of Adelaide \vspace{3pt} \\
   \quad \texttt{\{firstname.lastname\}@adelaide.edu.au} \\
   \quad {\tt\small \href{https://github.com/PaulAlbert31/RandLoRA}{https://github.com/PaulAlbert31/RandLoRA}}
}

\iclrfinalcopy 
\begin{document}

\maketitle

\begin{abstract}
Low-Rank Adaptation (LoRA) and its variants have shown impressive results in reducing the number of trainable parameters and memory requirements of large transformer networks while maintaining fine-tuning performance. 
The low-rank nature of the weight update inherently limits the representation power of fine-tuned models, however, thus potentially compromising performance on complex tasks.
This raises a critical question: when a performance gap between LoRA and standard fine-tuning is observed, is it due to the reduced number of trainable parameters or the rank deficiency?
This paper aims to answer this question by introducing \method{}, a parameter-efficient method that performs full-rank updates using a learned linear combinations of low-rank, non-trainable random matrices. Our method limits the number of trainable parameters by restricting optimization to diagonal scaling matrices applied to the fixed random matrices. This allows us to effectively overcome the low-rank limitations while maintaining parameter and memory efficiency during training.
Through extensive experimentation across vision, language, and vision-language benchmarks, we systematically evaluate the limitations of LoRA and existing random basis methods.
Our findings reveal that full-rank updates are beneficial across vision and language tasks individually, and even more so for vision-language tasks, where \method{} significantly reduces---and sometimes eliminates---the performance gap between standard fine-tuning and LoRA, demonstrating its efficacy.
\end{abstract}

\section{Introduction}

Large pre-trained models that leverage broad data have demonstrated significantly improved generalization capabilities and remarkable versatility across diverse tasks.
However, the resultant high parameter count also leads to a significant increase in the computational resources required to fine-tune such models on downstream tasks.
To tackle this issue, parameter-efficient fine-tuning (PEFT) approaches such as low-rank adaptation (LoRA)~\citep{2022_ICLR_lora},
draw inspiration from the low intrinsic dimensionality of pre-trained models~\citep{intrinsic_iclr_2018,instrinsic_dimensionality_2021}
and characterize the weight updates as the product of two low-rank matrices,
substantially reducing the number of trainable parameters and memory requirements during training.
This formulation leads to an adaptable number of trainable parameters, as one modifies the rank of the matrices, providing great flexibility under various resource constraints.

In spite of the strong performance of LoRAs in parameter-efficient settings,
our investigation uncovers an accuracy plateau, wherein an increase of rank and thus learnable parameters fail to bridge the accuracy gap with standard fine-tuning.
These undesirable scaling properties\blue{~\citep{2024_ICLR_VeRA}} raise questions about the inherent limitations imposed by the low-rank structure, particularly when tackling complex tasks that benefit from larger parameter counts.
This issue would ideally be addressed by introducing full-rank updates while maintaining the parameter-efficiency.
To this end, we propose \method{}, a PEFT method that leverages a set of linearly-independent random bases in the form of non-trainable low-rank matrices.
By solely learning scaling coefficients for the linear combination of the random low-rank bases, our method achieves full-rank updates, while maintaining low memory usage.
As a result, \method{} strikes a balance between parameter efficiency and full-rank updates, allowing for more flexible and effective fine-tuning.

Through extensive experimentation, we empirically demonstrate the limitations of the low-rank formulation in LoRA, particularly on vision-language tasks, and show how \method{} can improve performance under similar parameter budget.  Figure~\ref{fig:parameterefficient} summarizes our findings across pure vision (DinoV2), vision-language (CLIP) and commonsense reasoning (LLama3-8B), where increasing LoRA's parameter count has highly diminishing returns. \blue{We find that \method{} outperforms LoRA as the parameter budget expands, while remaining parameter efficient thanks to its full-rank update strategy.} We conclude our investigation with an insightful discussion on the distinctive characteristics of \method{} where our analysis reveals that, in contrast to LoRA, \method{} yields activation patterns in deeper layers that closely align with those obtained through full fine-tuning. Furthermore, our visualization of the loss landscape reveals that the local minima reached by \method{} is often closer to that reached by standard fine-tuning, and it always leads to a lower loss than LoRA for an equal parameter count. Additionally, we explore the integration of sparse random bases, where initial findings highlight that sparse bases preserves the performance of \method{}. This suggests promising avenues to further reduce memory and computational requirements when training large transformer models, without compromising model performance.

Our contributions are summarized as:
\begin{enumerate}
    \item We investigate the interplay between rank and number of trainable parameters when fine-tuning large pre-trained models, highlighting the limitations of LoRA in improving performance when larger ranks are required.    
    \item We propose \method{}, a novel parameter-efficient fine-tuning (PEFT) strategy based on random basis combinations, enabling full-rank updates without memory overhead over LoRA.
    \item We rigorously assess \method{} across diverse pre-trained architectures and tasks, spanning pure vision and vision-language image classification to commonsense reasoning, demonstrating its versatility and effectiveness.
\end{enumerate}    

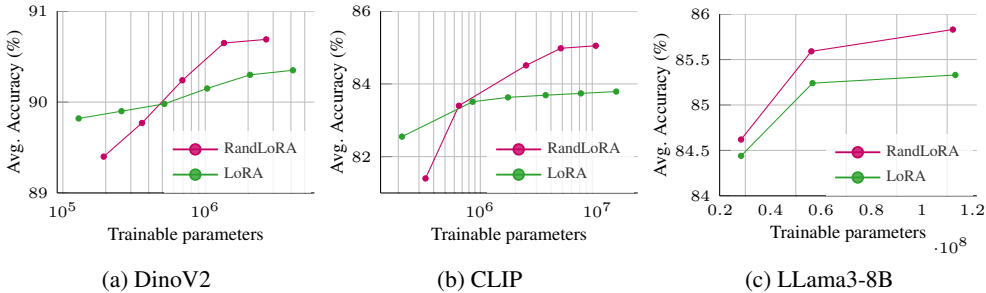
\begin{figure}
    \centering
    \subfloat[DinoV2]{\begin{tikzpicture}
    \begin{axis}[
        width=5cm,
        height=4cm,
        font=\tiny,
        ymin=89,
        ymax=91,
        legend style={draw=none,at={(0.55,1.3)}, 
                    text width=1.3cm,
                    anchor=north,
                    legend columns=6,
                    fill=none,
                nodes={scale=0.8, transform shape},},
        xlabel={\scriptsize{Trainable parameters}},
        ylabel={\scriptsize{Avg. Accuracy (\%)}},
        ymajorgrids=true,
        ytick={89,90,91},
        yticklabel={$\tiny\pgfmathprintnumber{\tick}$},
        ylabel shift=-5pt,
        xlabel shift=-3pt,
        axis x line*=bottom,
        axis y line*=left,  
        cycle list={color1,color2,color3,color4,color5,color6,color7,color8,color9,color21},
        minor x tick num=1,
        xminorgrids,
        minor tick length=0,
        major x tick style = transparent,
        mark size=1.0pt,
        xminorticks=true,
        xmode=log,
        log basis x=10,
        x tick label style={/pgf/number format/.cd,fixed,precision=2},
        ]

    \addplot+[mark=*,mark size=1pt,mark options={fill=color6},color=color6,draw opacity=1] coordinates {(193536, 89.40)(359424, 89.77)(691200, 90.24)(1354752, 90.65)(2681856, 90.69)};

    \addplot+[mark=*,mark size=1pt,mark options={fill=color3},color=color3,draw opacity=1] coordinates {(129024, 89.82)(258048, 89.90)(516096,89.98)(1032192,90.15)(2064384,90.30)(4128768,90.35)};

\end{axis}
    \begin{axis}[
       xmin=0,
       xmax=16,
       ymin=1,
       ymax=2,
       hide axis,
       width=5cm,
       height=5cm,
       font=\tiny,
       mark size=1.0pt,
       legend style={
               at={(0.7,0.0)},
               anchor=south,
               draw=none,
               legend columns=1,
               fill opacity=0.8,
               nodes={scale=1, transform shape},
               cells={align=left},
           },
       legend cell align={left},
   ]

\addplot+ [mark=*, mark size=2pt, mark options={fill=color6},color=color6, line width=0.7pt,solid] coordinates { (0,0) };
\addlegendentry{RandLoRA}

\addplot+ [mark=*, mark size=2pt, mark options={fill=color3},color=color3, line width=0.7pt,solid] coordinates { (0,0) };
\addlegendentry{LoRA}

\end{axis}
\end{tikzpicture}}
    \subfloat[CLIP]{\begin{tikzpicture}
    \begin{axis}[
        width=5cm,
        height=4cm,
        font=\tiny,
        ymin=81,
        ymax=86,
        legend style={draw=none,at={(0.55,1.3)}, 
                    text width=1.3cm,
                    anchor=north,
                    legend columns=6,
                    fill=none,
                nodes={scale=0.8, transform shape},},
        xlabel={\scriptsize{Trainable parameters}},
        ylabel={\scriptsize{Avg. Accuracy (\%)}},
        ymajorgrids=true,
        ytick={80,82,84,86},
        yticklabel={$\tiny\pgfmathprintnumber{\tick}$},
        ylabel shift=-5pt,
        xlabel shift=-3pt,
        axis x line*=bottom,
        axis y line*=left,  
        cycle list={color1,color2,color3,color4,color5,color6,color7,color8,color9,color21},
        minor x tick num=1,
        xminorgrids,
        minor tick length=0,
        major x tick style = transparent,
        mark size=1.0pt,
        xminorticks=true,
        xmode=log,
        log basis x=10,
        x tick label style={/pgf/number format/.cd,fixed,precision=2},
        ]

    \addplot+[mark=*,mark size=1pt,mark options={fill=color6},color=color6,draw opacity=1] coordinates {(341000, 81.4)(657000, 83.4)(2454000, 84.51)(4850000, 84.98)(9642000, 85.05)};

    \addplot+[mark=*,mark size=1pt,mark options={fill=color3},color=color3,draw opacity=1] coordinates {(215000, 82.55)(860000, 83.51)(1720000,83.63)(3605000,83.69)(7209000,83.74)(14418000,83.79)};

\end{axis}
    \begin{axis}[
       xmin=0,
       xmax=16,
       ymin=1,
       ymax=2,
       hide axis,
       width=5cm,
       height=5cm,
       font=\tiny,
       mark size=1.0pt,
       legend style={
               at={(0.7,0.0)},
               anchor=south,
               draw=none,
               legend columns=1,
               fill opacity=0.8,
               nodes={scale=1, transform shape},
               cells={align=left},
           },
       legend cell align={left},
   ]

\addplot+ [mark=*, mark size=2pt, mark options={fill=color6},color=color6, line width=0.7pt,solid] coordinates { (0,0) };
\addlegendentry{RandLoRA}

\addplot+ [mark=*, mark size=2pt, mark options={fill=color3},color=color3, line width=0.7pt,solid] coordinates { (0,0) };
\addlegendentry{LoRA}

\end{axis}
\end{tikzpicture}}
    \subfloat[LLama3-8B]{\begin{tikzpicture}
    \begin{axis}[
        width=5cm,
        height=4cm,
        font=\tiny,
        ymin=84,
        ymax=86,
        legend style={draw=none,at={(0.55,1.3)}, 
                    text width=1.3cm,
                    anchor=north,
                    legend columns=6,
                    fill=none,
                nodes={scale=0.8, transform shape},},
        xlabel={\scriptsize{Trainable parameters}},
        ylabel={\scriptsize{Avg. Accuracy (\%)}},
        ymajorgrids=true,
        ytick={84,84.5,85,85.5,86},
        yticklabel={$\tiny\pgfmathprintnumber{\tick}$},
        ylabel shift=-5pt,
        xlabel shift=-3pt,
        axis x line*=bottom,
        axis y line*=left,  
        cycle list={color1,color2,color3,color4,color5,color6,color7,color8,color9,color21},
        minor x tick num=1,
        xminorgrids,
        minor tick length=0,
        major x tick style = transparent,
        mark size=1.0pt,
        xminorticks=true,
        ]

    \addplot+[mark=*,mark size=1pt,mark options={fill=color6},color=color6,draw opacity=1] coordinates {(28309760, 84.62)(56162560, 85.59)(112262816, 85.83)};

    \addplot+[mark=*,mark size=1pt,mark options={fill=color3},color=color3,draw opacity=1] coordinates {(28311552, 84.44)(56623104, 85.24)(113246208, 85.33)};

\end{axis}
    \begin{axis}[
       xmin=0,
       xmax=16,
       ymin=1,
       ymax=2,
       hide axis,
       width=5cm,
       height=5cm,
       font=\tiny,
       mark size=1.0pt,
       legend style={
               at={(0.7,0.0)},
               anchor=south,
               draw=none,
               legend columns=1,
               fill opacity=0.8,
               nodes={scale=1, transform shape},
               cells={align=left},
           },
       legend cell align={left},
   ]

\addplot+ [mark=*, mark size=2pt, mark options={fill=color6},color=color6, line width=0.7pt,solid] coordinates { (0,0) };
\addlegendentry{RandLoRA}

\addplot+ [mark=*, mark size=2pt, mark options={fill=color3},color=color3, line width=0.7pt,solid] coordinates { (0,0) };
\addlegendentry{LoRA}

\end{axis}
\end{tikzpicture}}
    \caption{LoRA becomes limited by the rank of its update. We train DinoV2 and CLIP to classify 21 image datasets and LLama3-8B to solve 8 commonsense reasoning tasks.}
    \label{fig:parameterefficient}
\end{figure}

\section{Related work\label{sec:related}}

\subsection{Low Rank Adaptation of Large Models}

Low Rank Adaptation (LoRA) of large language models has revolutionized the fine-tuning paradigm, enabling memory-constrained adaptation to specialist tasks and democratizing access to larger models. Initially introduced by~\citep{2022_ICLR_lora}, LoRA leverages the observation that weight updates during fine-tuning can converge to suitable performances without necessitating full rank updates. By factorizing weight updates into the product of two low rank matrices, LoRA achieves a memory-efficient solution for adapting large models. Moreover, once the low rank matrices are merged into the original weight matrix size, no latency is present during inference.
Several improvements have been proposed to build upon LoRA's success. Weight-decomposed LoRAs (DoRA)~\citep{2024_ICML_DoRA} proposes to improve convergence by decomposing LoRA updates into magnitude and direction components. AdaLoRA~\citep{2023_ICLR_AdaLoRA} and AutoLoRA~\citep{2024_arxiv_autolora}, utilize specialized metrics or meta-learning to propose rank-adapted LoRA formulations that dynamically adjust the rank to suit every layer's need. \blue{Other improvements include initialization strategies for the low rank matrices using the truncated SVD of the pre-trained weights and where the whole decomposition is fine-tuned as in Pissa~\citep{2024_NeurIPS_Pissa} or where only the singular value matrix is as in SVFT~\citep{2024_ICMLW_SVFT} or LoRA-XS~\citep{2024_loraxs}. Further improvements are proposed in HydraLoRA~\citep{2024_NeurIPS_HydraLoRA} where the scaling-up matrix of the low rank decomposition is split into multiple ones with a routing layer added to select the contribution of each head. This formulation enhances multi-task learning at the cost of losing the merging capabilities of LoRA in the pre-trained weight at test-time.}
These advancements collectively enhance the efficiency of LoRA, solidifying its position as a cornerstone of large language model fine-tuning.

\subsection{Parameter-Efficient fine-tuning (PEFT) using Random Bases}
Recent research has focused on further reducing the trainable parameter count of LoRA, a crucial aspect for low-shot applications where minimizing trainable parameters can prevent overfitting and enhance generalization. A promising direction involves utilizing random bases combinations, where randomly generated matrices are combined using a limited number of trainable parameters to estimate a weight update.

PRANC~\citep{2023_ICCV_PRANC} pioneered the random base strategy by learning a weighted averaged of random matrices through back-propagation. PRANC's solution averages multiple full size weight matrices for each layer, leading to high memory consumption. To address this, the authors generate random bases on the fly during forward and backward passes using a fixed seed random number generator, reducing memory usage to that of the largest trained layer in the network at the cost of training latency.

Building upon PRANC, NOLA~\citep{2024_ICLR_NoLA} introduces an improved algorithm where random bases are estimated as the product of two low-rank random matrices, each weighed using a learnable scalar and summed before matrix multiplication. This approach effectively approximates a rank 1 LoRA with significantly fewer trainable parameters and largely reduces memory consumption during training over PRANC.

Concurrently, VeRA~\citep{2024_ICLR_VeRA} proposed an alternative strategy utilizing a single high-rank random matrix (typically 256 or 1024), instead of summing multiple rank 1 matrices as in NoLA. VeRA also employs a scaling strategy of random bases distinct from NoLA, detailed in section~\ref{sec:theorymethod}, which relates to our approach. Both NOLA and VeRA achieve comparable performance to LoRA in few-shot fine-tuning scenarios while training substantially fewer parameters.

\subsection{Alternative strategies for parameter-efficient fine-tuning}
We report here on alternatives to weight tuning for parameter-efficient adaptation, specifically focusing on prompt tuning.  Context Optimization (CoOP) \citep{2022_IJCV_CoOP} introduced learnable context vectors for CLIP class names, later generalized to instance-specific prompts in Conditional CoOP (CoCoOP) \citep{2022_CVPR_CoCoOP}. Recent prompt tuning methods, like DePT \citep{2024_CVPR_DEPT} and PromptSRC \citep{2023_CVPR_PromptSRC}, emphasize knowledge preservation by isolating shared subspaces or regularizing prompts.  While parameter-efficient, prompt tuning can struggle with generalization beyond few-shot settings \citep{2024_arxiv_peftsurvey} and may be less effective than LoRA as data increases \citep{2024_CVPR_promptvslora}. We therefore consider prompt tuning orthogonal to weight-tuning for the scope of this paper and exclude it from direct RandLoRA comparisons except for early results found in Appendix \ref{app;deptresults}.

\section{Motivations}

Our literature review reveals that research on improving LoRA is focused on reducing the number of trainable parameters further, either through adaptable ranks or by using fixed or shared low rank projection matrices. When looking at moderate to larger parameter budgets however LoRA remains highly competitive.

We identify that early research has convincingly demonstrated the promise of random basis combinations as a parameter-efficient strategy for large models, particularly in few-shot scenarios. Two approaches have emerged, each representing a distinct paradigm. VeRA advocates for a unique random base with large rank, while NoLA proposes to average a large number of random bases with small ranks. Both approaches report performance comparable to LoRA in few-shot scenarios while converging on a significantly reduced number of trainable parameters. However, as we will demonstrate, this reduction comes at the cost of limited performance when venturing beyond few-shot learning, limiting the scalability of these algorithms.

Finally, we report that LoRA is predicated on the assumption that low-rank updates suffice for fine-tuning large models. We aim in this paper to question the universality of this hypothesis, exploring scenarios where full rank alternatives may be necessary. The fundamental question follows: is parameter efficiency achieved through low-rank approximation limited by (1) the low-rank nature of the update or (2) by the low parameter count. Can parameter-efficient full rank updates provide a more accurate solution ? 
This paper aims to address these questions, exploring the balance between parameter efficiency and low-rank fine-tuning of large transformer models, and shedding light on the limitations of existing approaches.

\section{\method{}---parameter-efficient fine-tuning with full rank\label{sec:theorymethod}}
\subsection{Weight updates as a sum of low-rank matrices}

Let $W_0 \in \mathbb{R}^{D \times d}$ be a weight matrix of a large pre-trained model.
Fine-tuning aims to find an appropriate $\Delta W \in \mathbb{R}^{D\times d}$,
such that the fine-tuned weights $W_0 + \Delta W$ lead to an adapted model, tailored to a specific downstream task.
Without loss of generality, let us assume $d < D$.
The motivation behind \method{} stems from the singular value decomposition (SVD) of $\Delta W$, i.e., $\Delta W=U\Sigma V^\mathsf{T}$, where $U \in \mathbb{R}^{D \times d}$, $\Sigma \in \mathbb{R}^{d \times d}$, $V \in \mathbb{R}^{d \times d}$.
This decomposition can be written as the sum of the product of rank-one matrices, as follows
\begin{equation}
 \Delta W = \sum_{i=1}^d \mathbf{u}_i\sigma_i \mathbf{v}_i^\mathsf{T},
\label{eqn;svd_rank1_decomp}
\end{equation}
where $\mathbf{u}_i$ and $\mathbf{v}_i$ denote the columns of $U$ and $V$, respectively.
We suggest that in this context, low-rank updates such as LoRAs can be characterized as an approximation of the few largest singular values while the rest of the information in $\Delta W$ being discarded.
To better illustrate this point, let us denote the rank of LoRA by $r$ and for brevity of exposition, assume $d$ is divisible by $r$.
We rewrite~\eqref{eqn;svd_rank1_decomp} as a sum of the product of rank-$r$ matrices, as follows
\begin{align}\label{eqn;svd_rankk_decomp}
\Delta W &= \sum_{j=1}^{n} U_j \Sigma_j V_j^\mathsf{T}, 
\end{align}
where $U_j \Sigma_j V_j^\mathsf{T} = \sum_{i=rj}^{r(j+1)} \mathbf{u}_i\sigma_i \mathbf{v}_i^\mathsf{T}$ and where $n = d / r$.
This formulation reveals how LoRA models the approximates the first low-rank partition $U_1 \Sigma_1 V_1^\mathsf{T}$, and implicitly assumes $\sum_{j=2}^{n} U_j \Sigma_j V_j^\mathsf{T} \approx 0$.
We however argue that the remaining $n-1$ terms can play a crucial role when capturing more complex task-specific variations that require larger deviations from the pre-trained weight $W_0$.

\subsection{Parameter-efficient approximation of low-rank matrices\label{sec:weightupdate}}

Approximating more terms in the decomposition of $\Delta W$ using LoRA's formulation quickly becomes parameter inefficient, culminating to $Dd + d^2$ parameters for a full rank $d$ in place of the original $Dd$ parameters of $\Delta W$. To perform full-rank updates while maintaining parameter-efficiency, we propose instead to approximate each term of $\Delta W$ in \eqref{eqn;svd_rankk_decomp} using low-rank random bases where only scaling coefficients are learned,
\begin{align}
  \Delta W = \sum_{j=1}^n B_j \Lambda_j A_j \Gamma_j, \label{eq:Wfirst}
\end{align}
where $B_j \in \mathbb{R}^{D \times r}$ and $A_j \in \mathbb{R}^{r \times d}$ are non-trainable, random matrices.
The two learnable diagonal scaling matrices, $\Lambda_j \in \mathbb{R}^{r \times r}$ and $\Gamma_j \in \mathbb{R}^{d \times d}$ are unique to each of the $n$ terms and fulfill complementary roles to improve the approximation.
We aim for $A_j\Gamma_j$ transform the input features into an low-dimensional space (rank-$r$), $\Lambda_j$ to scale the compressed features which are then transformed back into the desired output space by $B_j$.\footnote{The formulation of our method is similar to that of VeRA~\citep{2024_ICLR_VeRA}, which will be discussed in detail in section~\ref{sec:diffvera}.}
Since $\Gamma_j$ operates on the column space of $A_j$ and is unique to each $A_j$,
we use a unique shared matrix $A \in \mathbb{R}^{r \times d}$ across all $n$ terms without loss of expressivity but reducing memory consumption.
With a shared $A$, we formulate the update as 
\begin{align}
  \Delta W = \sum_{j=1}^n B_j \Lambda_j A \Gamma_j. \label{eq:W}
\end{align}
To achieve a full-rank update, we set $n=d/r$, leading to $\frac{d}{r}(d + r) = d^2/r + d$ learnable parameters.
Note that unlike LoRA, the number of learnable parameters is inversely proportional to the rank of the random bases in \method{}, as increasing the rank of the bases leads to a reduction in trainable parameters while maintaining full rank.
In summary, \method{} trades-off approximation accuracy for scope, sacrificing a more precise representation of the individual SVD elements of $\Delta W$ to capture a larger portion of its singular value decomposition.

\subsection{Convergence analysis\label{sec:maths}}
In this section, we present a theorem showing that weight updates using \method{} is an accurate approximation of general matrices under certain theoretical conditions.

\begin{theorem}\label{thm;random_sum_est_main1}
Let $W$ be a fixed $D \times d$ matrix, with $D > d$ and $\text{rank}(W) = d$.
Fix $1 \leq n \leq d$, such that $d = nr$.
The matrix $W$ can be factorized using SVD as
\begin{equation}
    W = \sum_j^{n} U_j \Sigma_j V_j^\mathsf{T},
\end{equation}
where $U_j \in \mathbb{R}^{D\times r}$, $V_j \in \mathbb{R}^{r \times d}$ are partitions of the left and right singular vectors,
and $\Sigma_j \in \mathbb{R}^{r \times r}$ contains r singular values.
For each $1 \leq j \leq n$, let $B_j$ denote a random $D\times r$ matrix whose entries are drawn i.i.d from either a Gaussian or uniform distribution, $A_j$ denotes an 
$r \times d$ matrix whose entries are drawn similarly, 
$\Lambda_j$ is a diagonal $r \times r$ matrix and $\Gamma_j$ is a diagonal 
$d \times d$ matrix drawn similarly. Assume 
\begin{equation}\label{eqn;frob_est_assump1}
    \lVert U_j\Sigma_jV^\mathsf{T}_j - 
    B_j\Lambda_jA_j\Gamma_j\rVert_F \leq \epsilon
\end{equation}
for each $1 \leq j \leq n$ for some $0 < \epsilon$.
Then we have that with probability $1$ that each $B_j\Lambda_jA_j\Gamma_j$ has full rank and
\begin{equation}
    \left\lVert W - \sum_{j=1}^nB_j\Lambda_jA_j\Gamma_j \right\rVert_F \leq 
    n \cdot \epsilon.
\end{equation}
\end{theorem}
For details on the proof of theorem \ref{thm;random_sum_est_main1} please refer to appendix \ref{app:proof_theo}.

Theorem~\ref{thm;random_sum_est_main1} is premised on $B_j\Lambda_jA_j\Gamma_j$ being a good approximation for the $r$-truncated singular value of $\Delta W$, which is shown to be true empirically in VeRA~\citep{2024_ICLR_VeRA} for example.
We show in this case that $\Delta W$ can be accurately approximated as $\sum_{j=1}^nB_j\Lambda_jA_j\Gamma_j$, motivating \method{}'s formulation.
In contrast, since the best approximation a rank-$r$ LoRA can achieve is the $r$-truncated SVD of $W$, then by Eckart-Young-Mirsky theorem, the Frobenius norm of the difference between $W$ and low-rank adaptation $BA$ is lower bounded as follows

\begin{align}
\lVert W - BA\rVert_F &\geq \left\lVert W - \sum_{i=1}^r \mathbf{u}_i \sigma_i \mathbf{v}_i^\mathsf{T} \right\rVert_F = \sum_{i=r+1}^d\sigma_i^2.  
\label{eq:lora_bound}
\end{align}

We conclude that while LoRA's rank $r$ approximation is limited by the sum of the last $d-r-1$ squared singular values of $W$, \method{} does not present this low bound and is only limited by how close ($\epsilon$) can $B_j\Lambda_j A_j\Gamma_j$ approximate length-r segments of the SVD of $W$.

\section{Experiments\label{sec:exp}}
\subsection{Experimental Settings}

We conduct a comprehensive comparison with three state-of-the-art approaches: LoRA~\citep{2022_ICLR_lora}, NoLA~\citep{2024_ICLR_NoLA}, and VeRA~\citep{2024_ICLR_VeRA}. We perform a hyper-parameter search to identify optimal settings for LoRA, NoLA, VeRA, and RandLoRA to ensure a fair comparison. More details about the experimental settings can be found in appendix~\ref{sec:appimplem}. \blue{Additional experiments on the General Language Understanding Evaluation (GLUE)~\citep{2019_ICLR_glue} and End-to-end (E2E)~\cite{2017_ACL_e2e} natural language generation benchmarks as well as further comparison with prompt-tuning algorithms are available in appendix~\ref{app;resultsglue}}.

\subsection{Vision: DinoV2 and CLIP's vision backbone\label{sec:expvision}}
We evaluate fine-tuning vision backbones for image classification using pre-trained ViT-B/14 DinoV2~\citep{2023_arXiv_dinov2} and ViT-B/32, ViT-L/14 CLIP~\citep{2021_ICML_CLIP} vision only backbones. We fine-tune on 21 datasets (Appendix~\ref{app:classifdatasets}, Table~\ref{tab:datasets}) and evaluate \{1, 2, 4, 16\}-shot learning and performance with 50\% and 100\% training data.

We compare \method{} to LoRA rank 32 where \method{}'s rank is adjusted to match LoRA's parameters, and include VeRA and NoLA as random base alternatives. We fine-tune the vision backbones and learn linear classifiers for DinoV2, or use frozen CLIP language embeddings for classification. Results are displayed in Figure~\ref{fig:results-clip-vision} where we also report VRAM usage, detailed results are available in Appendix~\ref{app:dinoresults}.

We find that LoRA exhibits a smaller accuracy gap with standard fine-tuning (FT) on DinoV2 than CLIP.  With equal parameters, \method{} improves over LoRA, bridging the FT gap in both cases. We believe that LoRA's success on the DinoV2 backbone is partly explained by its training objective (see Section~\ref{sec:simactiv}). \method{} demonstrates LoRA's rank limitation for CLIP architectures and the benefit of full-rank updates in matching FT performance. VeRA and NoLA are efficient in few-shot settings but become limited with more data. 

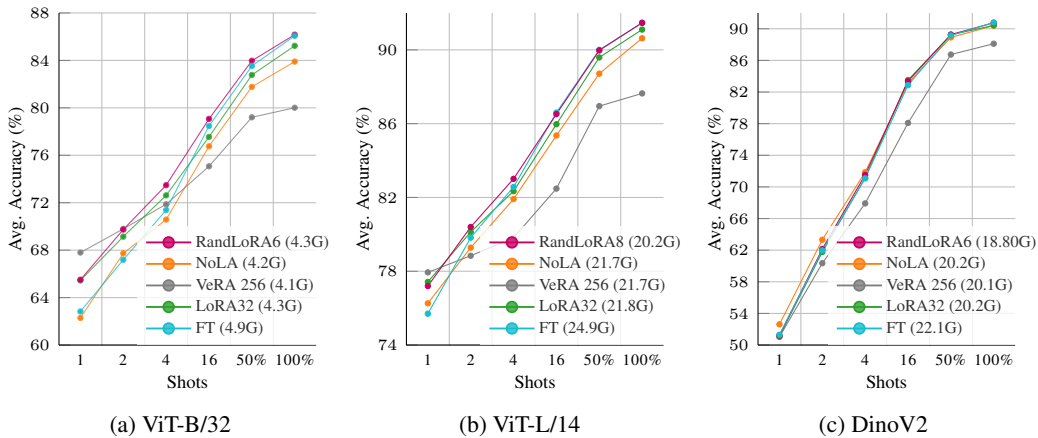
\begin{figure}
    \centering
    \subfloat[ViT-B/32]{\begin{tikzpicture}
    \begin{axis}[
        width=5cm,
        height=6cm,
        font=\tiny,
        ymin=60,
        ymax=88,
        legend style={draw=none,at={(0.55,1.3)}, 
                    text width=1.3cm,
                    anchor=north,
                    legend columns=6,
                    fill=none,
                nodes={scale=0.8, transform shape},},
        xlabel={\scriptsize{Shots}},
        ylabel={\scriptsize{Avg. Accuracy (\%)}},
        ymajorgrids=true,
        ytick={60,64,68,72,76,80,84,88},
        yticklabel={$\tiny\pgfmathprintnumber{\tick}$},
        ylabel shift=-5pt,
        xlabel shift=-3pt,
        axis x line*=bottom,
        axis y line*=left,  
        symbolic x coords={1,2,4,16,50\%,100\%},
        cycle list={color1,color2,color3,color4,color5,color6,color7,color8,color9,color21},
        xtick={1,2,4,16,50\%,100\%},
        xticklabel={\textsc{\tick}},
        minor x tick num=1,
        xminorgrids,
        minor tick length=0,
        major x tick style = transparent,
        mark size=1.0pt,
        ]

\addplot+[mark=*,mark size=1pt,mark options={fill=color2},color=color2,draw opacity=0.8] coordinates {(1, 62.29)(2, 67.73)(4, 70.58)(16, 76.76)(50\%, 81.77)(100\%, 83.90)};

\addplot+[mark=*,mark size=1pt,mark options={fill=color8},color=color8,draw opacity=0.8] coordinates {(1, 67.80)(2, 69.79)(4, 71.87)(16, 75.07)(50\%, 79.20)(100\%, 80.01)};

\addplot+[mark=*,mark size=1pt,mark options={fill=color3},color=color3,draw opacity=0.8] coordinates {(1, 65.45)(2, 69.14)(4, 72.62)(16, 77.55)(50\%, 82.77)(100\%, 85.23)};

\addplot+[mark=*,mark size=1pt,mark options={fill=color6},color=color6,draw opacity=0.8] coordinates {(1, 65.51)(2, 69.74)(4, 73.48)(16, 79.07)(50\%, 83.97)(100\%, 86.17)};

\addplot+[mark=*,mark size=1pt,mark options={fill=color10},color=color10,draw opacity=0.8] coordinates {(1, 62.83)(2, 67.19)(4, 71.38)(16, 78.47)(50\%, 83.53)(100\%, 86.07)};

\end{axis}
    \begin{axis}[
       xmin=0,
       xmax=16,
       ymin=1,
       ymax=2,
       hide axis,
       width=5cm,
       height=5cm,
       font=\small,
       mark size=1.0pt,
       legend style={
               at={(0.7,0.0)},
               anchor=south,
               draw=none,
               legend columns=1,
               fill opacity=0.8,
               nodes={scale=0.7, transform shape},
               cells={align=left},
           },
       legend cell align={left},
   ]

\addplot+ [mark=*, mark size=2pt, mark options={fill=color6},color=color6, line width=0.7pt,solid] coordinates { (0,0) };
\addlegendentry{RandLoRA6 (4.3G)}

\addplot+ [mark=*, mark size=2pt, mark options={fill=color2},color=color2, line width=0.7pt,solid] coordinates { (0,0) };
\addlegendentry{NoLA (4.2G)}

\addplot+ [mark=*, mark size=2pt, mark options={fill=color8},color=color8, line width=0.7pt,solid] coordinates { (0,0) };
\addlegendentry{VeRA 256 (4.1G)}

\addplot+ [mark=*, mark size=2pt, mark options={fill=color3},color=color3, line width=0.7pt,solid] coordinates { (0,0) };
\addlegendentry{LoRA32 (4.3G)}

\addplot+ [mark=*, mark size=2pt, mark options={fill=color10},color=color10, line width=0.7pt,solid] coordinates { (0,0) };
\addlegendentry{FT (4.9G)}

\end{axis}
\end{tikzpicture}}   
    \subfloat[ViT-L/14]{\begin{tikzpicture}
    \begin{axis}[
        width=5cm,
        height=6cm,
        font=\tiny,
        ymin=74,
        ymax=92,
        legend style={draw=none,at={(0.55,1.3)}, 
                    text width=1.3cm,
                    anchor=north,
                    legend columns=6,
                    fill=none,
                nodes={scale=0.8, transform shape},},
        xlabel={\scriptsize{Shots}},
        ylabel={\scriptsize{Avg. Accuracy (\%)}},
        ymajorgrids=true,
        ytick={74, 78, 82, 86, 90},
        yticklabel={$\tiny\pgfmathprintnumber{\tick}$},
        ylabel shift=-5pt,
        xlabel shift=-3pt,
        axis x line*=bottom,
        axis y line*=left,  
        symbolic x coords={0,1,2,4,16,50\%,100\%},
        cycle list={color1,color2,color3,color4,color5,color6,color7,color8,color9,color21},
        xtick={0,1,2,4,16,50\%,100\%},
        xticklabel={\textsc{\tick}},
        minor x tick num=1,
        xminorgrids,
        minor tick length=0,
        major x tick style = transparent,
        mark size=1.0pt,
        ]

\addplot+[mark=*,mark size=1pt,mark options={fill=color2},color=color2,draw opacity=1] coordinates {(1, 76.26)(2, 79.28)(4, 81.91)(16, 85.36)(50\%, 88.71)(100\%, 90.63)};

\addplot+[mark=*,mark size=1pt,mark options={fill=color8},color=color8,draw opacity=1] coordinates {(1, 77.94)(2, 78.84)(4, 79.82)(16, 82.48)(50\%, 86.95)(100\%, 87.64)};

\addplot+[mark=*,mark size=1pt,mark options={fill=color3},color=color3,draw opacity=1] coordinates {(1, 77.42)(2, 80.13)(4, 82.34)(16, 85.97)(50\%, 89.59)(100\%, 91.10)};

\addplot+[mark=*,mark size=1pt,mark options={fill=color10},color=color10,draw opacity=1] coordinates {(1, 75.70)(2, 79.82)(4, 82.57)(16, 86.60)(50\%, 90.00)(100\%, 91.46)};

\addplot+[mark=*,mark size=1pt,mark options={fill=color6},color=color6,draw opacity=1] coordinates {(1, 77.20)(2, 80.40)(4, 83.01)(16, 86.52)(50\%, 89.97)(100\%, 91.47)};

\end{axis}
    \begin{axis}[
       xmin=0,
       xmax=16,
       ymin=1,
       ymax=2,
       hide axis,
       width=5cm,
       height=5cm,
       font=\small,
       mark size=1.0pt,
       legend style={
               at={(0.7,0.0)},
               anchor=south,
               draw=none,
               legend columns=1,
               fill opacity=0.8,
               nodes={scale=0.7, transform shape},
               cells={align=left},
           },
       legend cell align={left},
   ]

\addplot+ [mark=*, mark size=2pt, mark options={fill=color6},color=color6, line width=0.7pt,solid] coordinates { (0,0) };
\addlegendentry{RandLoRA8 (20.2G)}

\addplot+ [mark=*, mark size=2pt, mark options={fill=color2},color=color2, line width=0.7pt,solid] coordinates { (0,0) };
\addlegendentry{NoLA (21.7G)}

\addplot+ [mark=*, mark size=2pt, mark options={fill=color8},color=color8, line width=0.7pt,solid] coordinates { (0,0) };
\addlegendentry{VeRA 256 (21.7G)}

\addplot+ [mark=*, mark size=2pt, mark options={fill=color3},color=color3, line width=0.7pt,solid] coordinates { (0,0) };
\addlegendentry{LoRA32 (21.8G)}

\addplot+ [mark=*, mark size=2pt, mark options={fill=color10},color=color10, line width=0.7pt,solid] coordinates { (0,0) };
\addlegendentry{FT (24.9G)}

   \end{axis}
\end{tikzpicture}}    
    \subfloat[DinoV2]{\begin{tikzpicture}
    \begin{axis}[
        width=5cm,
        height=6cm,
        font=\tiny,
        ymin=50,
        ymax=92,
        legend style={draw=none,at={(0.55,1.3)}, 
                    text width=1.3cm,
                    anchor=north,
                    legend columns=6,
                    fill=none,
                nodes={scale=0.8, transform shape},},
        xlabel={\scriptsize{Shots}},
        ylabel={\scriptsize{Avg. Accuracy (\%)}},
        ymajorgrids=true,
        ytick={50,54,58,62,66,70,74,78,82,86,90,94},
        yticklabel={$\tiny\pgfmathprintnumber{\tick}$},
        ylabel shift=-5pt,
        xlabel shift=-3pt,
        axis x line*=bottom,
        axis y line*=left,  
        symbolic x coords={1,2,4,16,50\%,100\%},
        cycle list={color1,color2,color3,color4,color5,color6,color7,color8,color9,color21},
        xtick={1,2,4,16,50\%,100\%},
        xticklabel={\textsc{\tick}},
        minor x tick num=1,
        xminorgrids,
        minor tick length=0,
        major x tick style = transparent,
        mark size=1.0pt,
        ]

\addplot+[mark=*,mark size=1pt,mark options={fill=color2},color=color2,draw opacity=1] coordinates {(1, 52.62)(2, 63.32)(4, 71.86)(16, 83.13)(50\%, 88.91)(100\%, 90.40)};

\addplot+[mark=*,mark size=1pt,mark options={fill=color8},color=color8,draw opacity=1] coordinates {(1, 51.03)(2, 60.35)(4, 67.91)(16, 78.09)(50\%, 86.74)(100\%, 88.11)};

\addplot+[mark=*,mark size=1pt,mark options={fill=color3},color=color3,draw opacity=1] coordinates {(1, 51.12)(2, 61.77)(4, 71.13)(16, 83.51)(50\%, 89.24)(100\%, 90.43)};

\addplot+[mark=*,mark size=1pt,mark options={fill=color6},color=color6,draw opacity=1] coordinates {(1, 51.22)(2, 62.15)(4, 71.52)(16, 83.38)(50\%, 89.27)(100\%, 90.76)};

\addplot+[mark=*,mark size=1pt,mark options={fill=color10},color=color10,draw opacity=1] coordinates {(1, 51.30)(2, 61.97)(4, 71.08)(16, 82.87)(50\%, 89.16)(100\%, 90.79)};

\end{axis}
    \begin{axis}[
       xmin=0,
       xmax=16,
       ymin=1,
       ymax=2,
       hide axis,
       width=5cm,
       height=5cm,
       font=\small,
       mark size=1.0pt,
       legend style={
               at={(0.7,0.0)},
               anchor=south,
               draw=none,
               legend columns=1,
               fill opacity=0.8,
               nodes={scale=0.7, transform shape},
               cells={align=left},
           },
       legend cell align={left},
   ]

\addplot+ [mark=*, mark size=2pt, mark options={fill=color6},color=color6, line width=0.7pt,solid] coordinates { (0,0) };
\addlegendentry{RandLoRA6 (18.80G)}

\addplot+ [mark=*, mark size=2pt, mark options={fill=color2},color=color2, line width=0.7pt,solid] coordinates { (0,0) };
\addlegendentry{NoLA (20.2G)}

\addplot+ [mark=*, mark size=2pt, mark options={fill=color8},color=color8, line width=0.7pt,solid] coordinates { (0,0) };
\addlegendentry{VeRA 256 (20.1G)}

\addplot+ [mark=*, mark size=2pt, mark options={fill=color3},color=color3, line width=0.7pt,solid] coordinates { (0,0) };
\addlegendentry{LoRA32 (20.2G)}

\addplot+ [mark=*, mark size=2pt, mark options={fill=color10},color=color10, line width=0.7pt,solid] coordinates { (0,0) };
\addlegendentry{FT (22.1G)}

\end{axis}
\end{tikzpicture}}
    \caption{Tuning CLIP and DinoV2 vision encoders for image classification. Accuracy averaged over 21 datasets. We additionally report max GPU VRAM usage during training.}
    \label{fig:results-clip-vision}
\end{figure}

\subsection{Vision-Language: CLIP}
We extend in this section our experimental setting to fine-tuning CLIP-like transformer architectures on classification datasets where contrary to section~\ref{sec:expvision} both the language and vision encoders of CLIP are trained. We add ImageNet~\citep{2012_NeurIPS_ImageNet} to the dataset pool to scale up to 22 classification datasets. To assess the effectiveness of \method{} compared to LoRA on models of varying sizes, we consider three variants of pre-trained CLIPs from the open-clip repository~\citep{2023_CVPR_openclip}: ViT-B/32 (151M parameters), ViT-L/14 (428M parameters) and ViT-H/14 (1B parameters). We scale the rank of the random bases in \method{} in the same way as section~\ref{sec:expvision} to maintain a number of parameters comparable to a rank 32 LoRA: \method{}-\{6,8,10\} for ViT-\{B/32,L/14,H/14\} respectively.

A summary of results is available in Figure~\ref{fig:results-clip} with detailed results being available in appendix~\ref{app:clipresults}. \blue{Because fine-tuning vision-language architectures such as CLIP is a harder optimization problem,} we observe the existence of a larger performance gap between full fine-tuning and LoRA than for pure vision, which we confirm is not bridged by increasing the rank of LoRA (see Figure~\ref{fig:parameterefficient}). This suggests that increasing parameter count is not enough, pointing towards the rank of the update as the possible limit to the performance of LoRA. When running \method{} with the same amount of trainable parameters, we observe that the gap with fine-tuning is bridged. When compared with NoLA and VeRA we come to the same conclusions as section~\ref{sec:expvision} although VeRA is this time much more competitive for larger data budgets, hinting towards the importance of high ranks for finetuning CLIP-like vision language architectures. We also report that our base sharing strategy allows \method{} to decrease VRAM usage over LoRA which can be relevant for large architectures such as ViT-H/14.

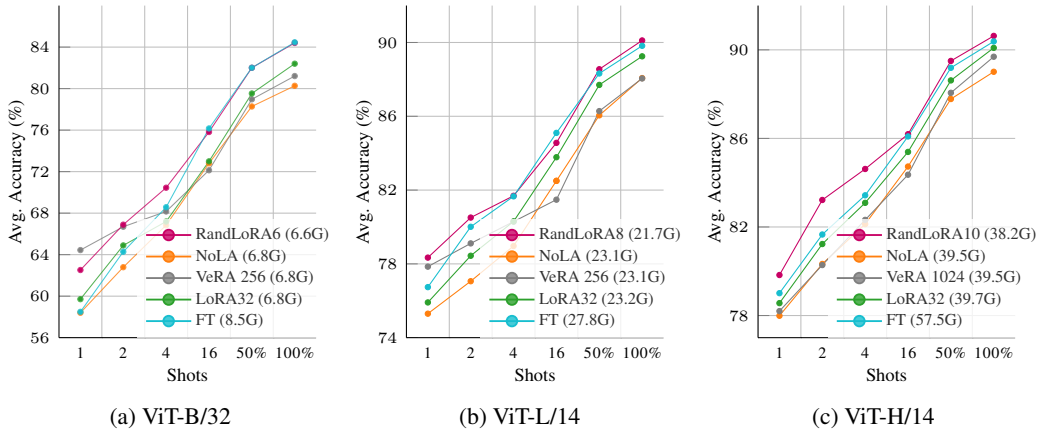
\begin{figure}
    \centering
    \subfloat[ViT-B/32]{\begin{tikzpicture}
    \begin{axis}[
        width=5cm,
        height=6cm,
        font=\tiny,
        ymin=56,
        ymax=88,
        legend style={draw=none,at={(0.55,1.3)}, 
                    text width=1.3cm,
                    anchor=north,
                    legend columns=6,
                    fill=none,
                nodes={scale=0.8, transform shape},},
        xlabel={\scriptsize{Shots}},
        ylabel={\scriptsize{Avg. Accuracy (\%)}},
        ymajorgrids=true,
        ytick={52,56,60,64,68,72,76,80,84},
        yticklabel={$\tiny\pgfmathprintnumber{\tick}$},
        ylabel shift=-5pt,
        xlabel shift=-3pt,
        axis x line*=bottom,
        axis y line*=left,  
        symbolic x coords={1,2,4,16,50\%,100\%},
        cycle list={color1,color2,color3,color4,color5,color6,color7,color8,color9,color21},
        xtick={1,2,4,16,50\%,100\%},
        xticklabel={\textsc{\tick}},
        minor x tick num=1,
        xminorgrids,
        minor tick length=0,
        major x tick style = transparent,
        mark size=1.0pt,
        ]

    \addplot+[mark=*,mark size=1pt,mark options={fill=color2},color=color2,draw opacity=0.8] coordinates {(1, 58.40)(2, 62.79)(4, 66.84)(16, 72.80)(50\%, 78.28)(100\%, 80.27)};

    \addplot+[mark=*,mark size=1pt,mark options={fill=color8},color=color8,draw opacity=0.8] coordinates {(1, 64.44)(2, 66.69)(4, 68.16)(16, 72.13)(50\%, 78.97)(100\%, 81.21)};

    \addplot+[mark=*,mark size=1pt,mark options={fill=color3},color=color3,draw opacity=0.8] coordinates {(1, 59.71)(2, 64.89)(4, 67.19)(16, 72.99)(50\%, 79.53)(100\%, 82.40)};

    \addplot+[mark=*,mark size=1pt,mark options={fill=color6},color=color6,draw opacity=0.8] coordinates {(1, 62.52)(2, 66.89)(4, 70.45)(16, 75.82)(50\%, 82.01)(100\%, 84.40)};
    
    \addplot+[mark=*,mark size=1pt,mark options={fill=color10},color=color10,draw opacity=0.8] coordinates {(1, 58.49)(2, 64.28)(4, 68.57)(16, 76.16)(50\%, 82.01)(100\%, 84.46)};

\end{axis}
    \begin{axis}[
       xmin=0,
       xmax=16,
       ymin=1,
       ymax=2,
       hide axis,
       width=5cm,
       height=5cm,
       font=\small,
       mark size=1.0pt,
       legend style={
               at={(0.7,0.0)},
               anchor=south,
               draw=none,
               legend columns=1,
               fill opacity=0.8,
               nodes={scale=0.7, transform shape},
               cells={align=left},
           },
       legend cell align={left},
   ]

\addplot+ [mark=*, mark size=2pt, mark options={fill=color6},color=color6, line width=0.7pt,solid] coordinates { (0,0) };
\addlegendentry{RandLoRA6 (6.6G)}

\addplot+ [mark=*, mark size=2pt, mark options={fill=color2},color=color2, line width=0.7pt,solid] coordinates { (0,0) };
\addlegendentry{NoLA (6.8G)}

\addplot+ [mark=*, mark size=2pt, mark options={fill=color8},color=color8, line width=0.7pt,solid] coordinates { (0,0) };
\addlegendentry{VeRA 256 (6.8G)}

\addplot+ [mark=*, mark size=2pt, mark options={fill=color3},color=color3, line width=0.7pt,solid] coordinates { (0,0) };
\addlegendentry{LoRA32 (6.8G)}

\addplot+ [mark=*, mark size=2pt, mark options={fill=color10},color=color10, line width=0.7pt,solid] coordinates { (0,0) };
\addlegendentry{FT (8.5G)}

\end{axis}
\end{tikzpicture}}   
    \subfloat[ViT-L/14]{\begin{tikzpicture}
    \begin{axis}[
        width=5cm,
        height=6cm,
        font=\tiny,
        ymin=74,
        ymax=92,
        legend style={draw=none,at={(0.55,1.3)}, 
                    text width=1.3cm,
                    anchor=north,
                    legend columns=6,
                    fill=none,
                nodes={scale=0.8, transform shape},},
        xlabel={\scriptsize{Shots}},
        ylabel={\scriptsize{Avg. Accuracy (\%)}},
        ymajorgrids=true,
        ytick={74, 78, 82, 86, 90},
        yticklabel={$\tiny\pgfmathprintnumber{\tick}$},
        ylabel shift=-5pt,
        xlabel shift=-3pt,
        axis x line*=bottom,
        axis y line*=left,  
        symbolic x coords={0,1,2,4,16,50\%,100\%},
        cycle list={color1,color2,color3,color4,color5,color6,color7,color8,color9,color21},
        xtick={0,1,2,4,16,50\%,100\%},
        xticklabel={\textsc{\tick}},
        minor x tick num=1,
        xminorgrids,
        minor tick length=0,
        major x tick style = transparent,
        mark size=1.0pt,
        ]

\addplot+[mark=*,mark size=1pt,mark options={fill=color2},color=color2,draw opacity=1] coordinates {(1, 75.30)(2, 77.06)(4, 78.97)(16, 82.50)(16, 82.50)(50\%, 86.05)(100\%, 88.07)};

\addplot+[mark=*,mark size=1pt,mark options={fill=color8},color=color8,draw opacity=1] coordinates {(1, 77.85)(2, 79.11)(4, 80.29)(16, 81.48)(50\%, 86.28)(100\%, 88.05)};

\addplot+[mark=*,mark size=1pt,mark options={fill=color3},color=color3,draw opacity=1] coordinates {(1, 75.91)(2, 78.44)(4, 80.30)(16, 83.78)(50\%, 87.70)(100\%, 89.25)};

\addplot+[mark=*,mark size=1pt,mark options={fill=color6},color=color6,draw opacity=1] coordinates {(1, 78.34)(2, 80.51)(4, 81.69)(16, 84.56)(50\%, 88.55)(100\%, 90.11)};

\addplot+[mark=*,mark size=1pt,mark options={fill=color10},color=color10,draw opacity=1] coordinates {(1, 76.74)(2, 80.01)(4, 81.66)(16, 85.10)(50\%, 88.32)(100\%, 89.82)};

\end{axis}
    \begin{axis}[
       xmin=0,
       xmax=16,
       ymin=1,
       ymax=2,
       hide axis,
       width=5cm,
       height=5cm,
       font=\small,
       mark size=1.0pt,
       legend style={
               at={(0.7,0.0)},
               anchor=south,
               draw=none,
               legend columns=1,
               fill opacity=0.8,
               nodes={scale=0.7, transform shape},
               cells={align=left},
           },
       legend cell align={left},
   ]

\addplot+ [mark=*, mark size=2pt, mark options={fill=color6},color=color6, line width=0.7pt,solid] coordinates { (0,0) };
\addlegendentry{RandLoRA8 (21.7G)}

\addplot+ [mark=*, mark size=2pt, mark options={fill=color2},color=color2, line width=0.7pt,solid] coordinates { (0,0) };
\addlegendentry{NoLA (23.1G)}

\addplot+ [mark=*, mark size=2pt, mark options={fill=color8},color=color8, line width=0.7pt,solid] coordinates { (0,0) };
\addlegendentry{VeRA 256 (23.1G)}

\addplot+ [mark=*, mark size=2pt, mark options={fill=color3},color=color3, line width=0.7pt,solid] coordinates { (0,0) };
\addlegendentry{LoRA32 (23.2G)}

\addplot+ [mark=*, mark size=2pt, mark options={fill=color10},color=color10, line width=0.7pt,solid] coordinates { (0,0) };
\addlegendentry{FT (27.8G)}

   \end{axis}
\end{tikzpicture}}    
    \subfloat[ViT-H/14]{\begin{tikzpicture}
    \begin{axis}[
        width=5cm,
        height=6cm,
        font=\tiny,
        ymin=77,
        ymax=92,
        legend style={draw=none,at={(0.55,1.3)}, 
                    text width=1.3cm,
                    anchor=north,
                    legend columns=6,
                    fill=none,
                nodes={scale=0.8, transform shape},},
        xlabel={\scriptsize{Shots}},
        ylabel={\scriptsize{Avg. Accuracy (\%)}},
        ymajorgrids=true,
        ytick={78, 82, 86, 90},
        yticklabel={$\tiny\pgfmathprintnumber{\tick}$},
        ylabel shift=-5pt,
        xlabel shift=-3pt,
        axis x line*=bottom,
        axis y line*=left,  
        symbolic x coords={0,1,2,4,16,50\%,100\%},
        cycle list={color1,color2,color3,color4,color5,color6,color7,color8,color9,color21},
        xtick={0,1,2,4,16,50\%,100\%},
        xticklabel={\textsc{\tick}},
        minor x tick num=1,
        xminorgrids,
        minor tick length=0,
        major x tick style = transparent,
        mark size=1.0pt,
        ]

\addplot+[mark=*,mark size=1pt,mark options={fill=color2},color=color2,draw opacity=1] coordinates {(1, 77.99)(2, 80.33)(4, 82.13)(16, 84.73)(50\%, 87.78)(100\%, 89.01)};

\addplot+[mark=*,mark size=1pt,mark options={fill=color8},color=color8,draw opacity=1] coordinates {(1, 78.20)(2, 80.28)(4, 82.32)(16, 84.36)(50\%, 88.06)(100\%, 89.69)};

\addplot+[mark=*,mark size=1pt,mark options={fill=color3},color=color3,draw opacity=1] coordinates {(1, 78.56)(2, 81.23)(4, 83.09)(16, 85.39)(50\%, 88.62)(100\%, 90.09)};

\addplot+[mark=*,mark size=1pt,mark options={fill=color6},color=color6,draw opacity=1] coordinates {(1, 79.83)(2, 83.22)(4, 84.62)(16, 86.19)(50\%, 89.50)(100\%, 90.63)};

\addplot+[mark=*,mark size=1pt,mark options={fill=color10},color=color10,draw opacity=1] coordinates {(1, 79.01)(2, 81.66)(4, 83.43)(16, 86.09)(50\%, 89.19)(100\%, 90.38)};

\end{axis}
    \begin{axis}[
       xmin=0,
       xmax=16,
       ymin=1,
       ymax=2,
       hide axis,
       width=5cm,
       height=5cm,
       font=\small,
       mark size=1.0pt,
       legend style={
               at={(0.7,0.0)},
               anchor=south,
               draw=none,
               legend columns=1,
               fill opacity=0.8,
               nodes={scale=0.7, transform shape},
               cells={align=left},
           },
       legend cell align={left},
   ]

\addplot+ [mark=*, mark size=2pt, mark options={fill=color6},color=color6, line width=0.7pt,solid] coordinates { (0,0) };
\addlegendentry{RandLoRA10 (38.2G)}

\addplot+ [mark=*, mark size=2pt, mark options={fill=color2},color=color2, line width=0.7pt,solid] coordinates { (0,0) };
\addlegendentry{NoLA (39.5G)}

\addplot+ [mark=*, mark size=2pt, mark options={fill=color8},color=color8, line width=0.7pt,solid] coordinates { (0,0) };
\addlegendentry{VeRA 1024 (39.5G)}

\addplot+ [mark=*, mark size=2pt, mark options={fill=color3},color=color3, line width=0.7pt,solid] coordinates { (0,0) };
\addlegendentry{LoRA32 (39.7G)}

\addplot+ [mark=*, mark size=2pt, mark options={fill=color10},color=color10, line width=0.7pt,solid] coordinates { (0,0) };
\addlegendentry{FT (57.5G)}

   \end{axis}
\end{tikzpicture}}
    \caption{Tuning CLIP's vision and language encoders for image classification. Accuracy averaged over 22 datasets. We additionally report max GPU VRAM usage during training.}
    \label{fig:results-clip}
\end{figure}

\subsection{Commonsense Reasoning}
We evaluate \method{} for fine-tuning LLMs on eight commonsense reasoning tasks (see Appendix~\ref{app:commonsensedatasets}). We fine-tune Qwen2 (0.5B), Phi3 (3B), and Llama3 (8B) models and assess data efficiency by training on both a 170,000-sample full dataset and a 15,000-sample subset, following~\citet{2023_arXiv_llmadapters}.

Table~\ref{tab:llmavgresults} compares \method{} to LoRA, VeRA, and NoLA. We test two LoRA ranks: rank-16 ("Efficient") and rank-32 ("Performant"). We then scale \method{} the same or lower amount of parameters to ensure a fair comparison. Detailed results are found in Appendix~\ref{tab:comsense}

\method{} performs competitively with, and sometimes surpasses, LoRA.  Phi3's strong zero-shot abilities enable VeRA and NoLA to achieve strong results despite fewer parameters. Conversely, Qwen2 and Llama3 require more adaptation, challenging VeRA and NoLA to match LoRA's performance. The 15k-sample regime can lead to overfitting when scaling trainable parameters for LoRA and \method{}, decreasing performance even with dropout regularization. When training on the full 170k samples, \method{} consistently outperforms LoRA. Results comparing with DoRA~\citep{2024_ICML_DoRA} for LLama3 only are available in Table~\ref{tab:llmavgresultsdora} in the appendix where \method{} outperforms both DoRA and LoRA for larger parameter budgets, while DoRA and LoRA are competitive at "Efficient" budgets.  We conclude \method{} is a compelling alternative to LoRA and DoRA for LLM fine-tuning, especially with larger datasets and parameter budgets.

\begin{table}[t]
\centering
\caption{Parameter-efficient fine-tuning of Large Language Models (LLMs). Results averaged over 8 commonsense reasoning tasks. We bold the best accuracy between parameter-equivalent \method{} and LoRA configurations.}
\setlength{\tabcolsep}{4pt}
\small
\begin{tabular}{lcccccccc}
\toprule
\multirow{2}{*}{Network} & \multirow{2}{*}{Size} & \multirow{2}{*}{ZeroShot} & \multirow{2}{*}{NoLA} & \multirow{2}{*}{VeRA} & \multicolumn{2}{c}{LoRA}  & \multicolumn{2}{c}{\method{}}  \\
\cmidrule(lr){6-7} \cmidrule(lr){8-9}
& & & & & \text{Efficient}  & \text{Performant} & \text{Efficient}  & \text{Performant} \\

\midrule
\multirow{2}{*}{Qwen2-0.5b} & 15k & 5.2 & 42.6 & 48.1 & 53.2 & 52.3 & \textbf{53.5}  &52.9 \\
& 170k & 5.2 & 47.4 & 51.8 & 57.4 & 57.3 & 57.7 & \textbf{57.9} \\
\midrule

\multirow{2}{*}{Phi3-3b} & 15k & 65.4 & 80.4 & 78.6 & 81.8 & 80.3 & 81.7 & \textbf{82.3} \\
& 170k & 65.4 & 82.3 & 81.4 & 84.6 & 85.0 & 84.7 & \textbf{85.2} \\
\midrule

\multirow{2}{*}{LLama3-8b} & 15k & 27.0 & 76.9 & 77.1 & 82.7 & \textbf{83.1} &  81.0 & 81.3  \\
& 170k & 27.0 & 81.2 & 81.7 & 84.4 & 85.2 & 84.6 & \textbf{85.6} \\
\bottomrule
\end{tabular}
\label{tab:llmavgresults}
\end{table}

\section{Discussion}

\begin{figure}[t]
    \centering
    \caption{How close do \method{} and LoRA get to standard fine-tuning ? We compare CKA scores of \method{} and LoRA with fine-tuned activations (top) and the mode connectivity in the loss landscape of UCF101 (bottom)}
    \begin{subfigure}[b]{0.45\textwidth}
    
        \begin{subfigure}[b]{\textwidth}
        \includegraphics[width=\textwidth]{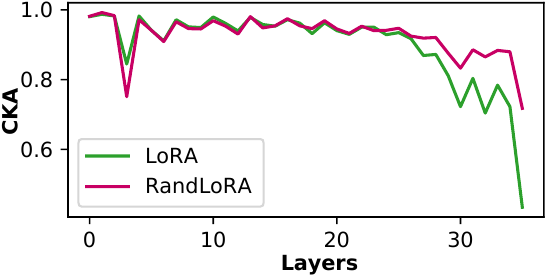}
        \caption{CKA with fine-tuning CLIP}
        \label{fig:sub1}
        \end{subfigure}
        \begin{subfigure}[b]{\textwidth}
        \includegraphics[width=\textwidth]{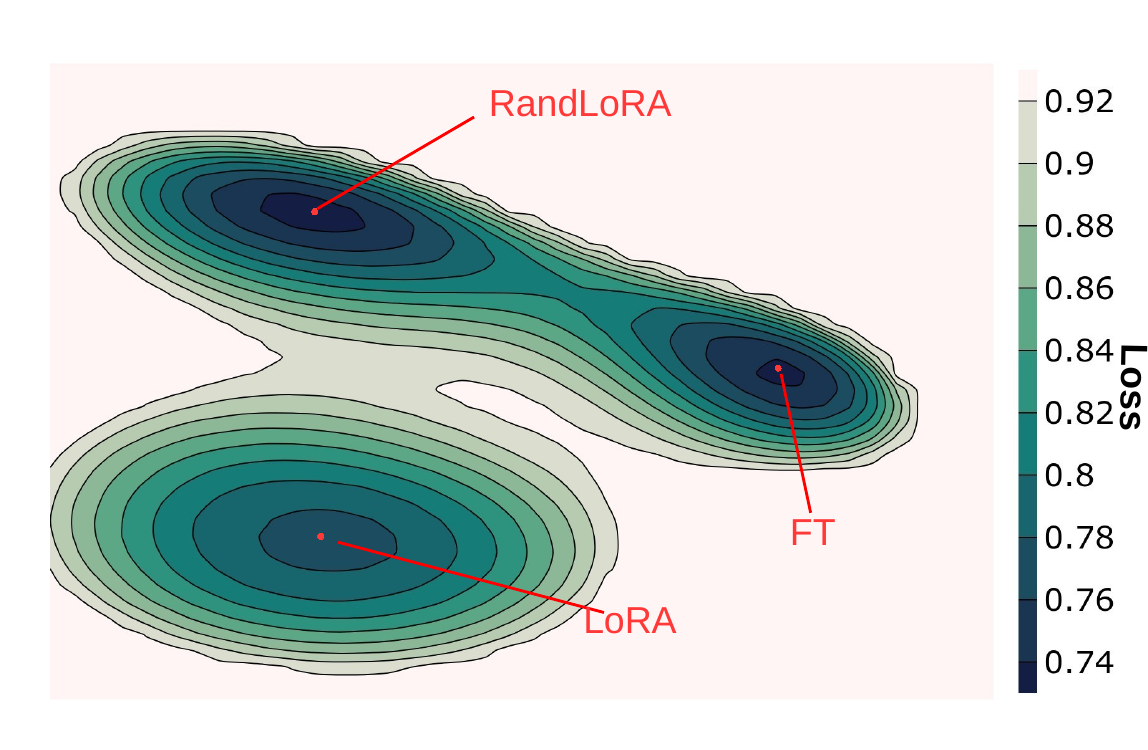}
        \caption{Loss landscape CLIP}
        \label{fig:sub2}
    \end{subfigure}
    \end{subfigure}    
    \quad
    \begin{subfigure}[b]{0.45\textwidth}

    \begin{subfigure}[b]{\textwidth}
        \includegraphics[width=\textwidth]{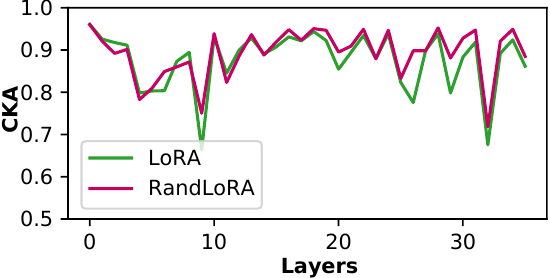}
        \caption{CKA with fine-tuning DinoV2}
        \label{fig:s1}
        \end{subfigure}
    
    \begin{subfigure}[b]{\textwidth}
        \includegraphics[width=\textwidth]{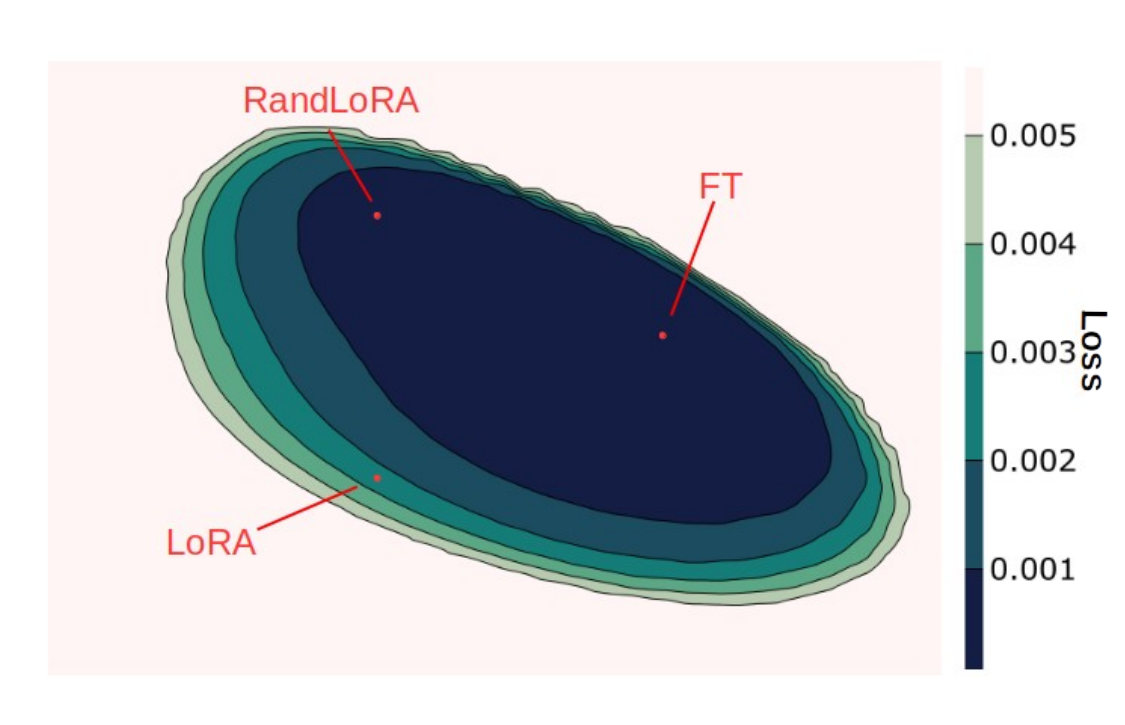}
        \caption{Loss landscape DinoV2}
        \label{fig:s2}
    \end{subfigure}
    \end{subfigure}
    \label{fig:main}
\end{figure}

\subsection{Similarities with fine-tuning: activations\label{sec:simactiv}}
We evaluate activation similarity to assess LoRA and \method{}'s ability to mimic fine-tuned model activations. Using the Centered Kernel Alignment (CKA)~\citep{2019_ICML_CKA} metric, we measure the similarity between activations of LoRA, \method{}, and a fully fine-tuned model. This protocol assesses how well each method captures dataset-specific activation patterns. Figure~\ref{fig:sub1} shows CKA scores for self-attention and MLP layers in CLIP and DinoV2 vision backbones, averaged over 5 datasets where \method{} imrpoves over LoRA. For CLIP, LoRA's CKA decreases in deeper layers, losing alignment with fine-tuned activations.  \method{}, with equal parameters, matches LoRA's early layer alignment but improves upon it in deeper layers.  This CKA drop for LoRA in deeper layers is absent in DinoV2, explaining LoRA's near-identical accuracy to fine-tuning on DinoV2. This difference likely arises from training objectives: DinoV2's visual objective creates classification-ready features needing minimal weight adjustments, thus low-rank LoRA suffices. CLIP's multimodal objective, however, demands higher ranks for effective adaptation to vision tasks.

\subsection{Similarities with fine-tuning: loss landscape}
We analyze loss landscape connectivity for models fine-tuned with standard fine-tuning, LoRA, and \method{}.  We visualize a 2D loss landscape plane by positioning LoRA, \method{}, and fine-tuning models at (0,0), (1,0), and (0.5,1) respectively. For each point $(x, y)$ on this plane, we interpolate model weights by solving for coefficients $\alpha_i$ (where $\sum_{i=1}^3\alpha_i = 1$) and evaluate the interpolated model's loss on a $5\%$ training subset.
Figure~\ref{fig:sub2} shows that for CLIP, \method{} reaches a deeper loss minima than LoRA, often with a low-loss path to the fine-tuning optimum, and despite training the same parameter count. For DinoV2, all optima reside in a shared low-loss basin, with LoRA already close to fine-tuning, reflecting LoRA's strong performance on this task. These visualizations reinforce LoRA's low rank it particularly limiting for complex tasks, and demonstrate \method{}'s ability to achieve deeper minima than LoRA with equal parameters due to full-rank updates.  Appendix~\ref{app:losslandscape} provides 3D visualizations for additional datasets.

\subsection{Further studies on full vs low rank fine-tuning of CLIP}
We investigate whether \method{}'s CLIP performance advantage over LoRA stems from better SVD approximation or its full-rank capability. We ablate \method{} with two rank-controlled variants. \method{}-a restricts the update rank to $r$ by averaging bases before multiplication: $\Delta W = \left(\sum_{i=1}^N B_i \Lambda_i \right) \left(\sum_{i=1}^N A_i \Gamma_i\right)$. \method{}-b uses half-rank updates by setting $N = \text{rank}(\Delta W)/r/2$ and adjusting base rank to maintain parameter count parity with \method{}-$r$. All variants train the same parameters, only update rank varies. Table~\ref{tab:abla} presents accuracy on 100\% of 22 datasets for CLIP ViT-B/32. Results show that higher update rank correlates with better performance, given equal parameter counts. This supports the importance of large rank updates, particularly for CLIP fine-tuning.

\begin{table}[t]
\centering
\begin{minipage}{0.4\textwidth}
\centering
\caption{Ablation on the rank of the updates. The same amount of trainable parameters is used in all methods.}
\begin{tabular}{lcc}
\toprule
Method & Rank & Accuracy \\
\midrule
LoRA & 32 & 83.74 \\
\method{}-a & 32 & 83.62 \\
\method{}-b & 384 & 85.32 \\
\method{}-6 & 768 & 85.98 \\
\bottomrule
\end{tabular}

\label{tab:abla}
\end{minipage}
\quad
\begin{minipage}{0.55\textwidth}
\centering
\caption{Fine-tuning CLIP or LLama3 using \method{} \blue{different random distributions or base sparsity.}}
\begin{tabular}{lcc}
\toprule
Model & Sparsity & Accuracy \\
\midrule
CLIP-ViT-B/32 - uniform & 0\% & 85.98 \\
\blue{CLIP-ViT-B/32 - normal} & 0\% & 85.61 \\
\blue{CLIP-ViT-B/32 - binary} & 0\% & 85.52 \\
CLIP-ViT-B/32 & 66\% & 85.43 \\
\blue{CLIP-ViT-B/32} & 93\% & 85.57 \\
\blue{CLIP-ViT-B/32} & 98\% & 84.35\\
\blue{CLIP-ViT-B/32} & 99\% & 83.34\\
\midrule
LLama3-8b & 0\% & 85.59 \\
LLama3-8b & 66\% & 85.42 \\
\bottomrule
\end{tabular}

\label{tab:ablasparse}
\end{minipage}
\end{table}

\subsection{Sparse random matrices\label{sec:randomprojtheory}}
We propose to investigate using sparse random matrices for improved memory and computational efficiency, drawing inspiration from random projection literature and the Johnson-Lindenstrauss lemma~\citep{1984_JLlemma}. We adopt the sparse construction from~\citet{2001_SIFKDD_sparserandom} and~\citet{2006_ACM_verysparse}, where matrix elements are $\{-1, 0, 1\}$ with probabilities $\{\frac{1}{s}, 1-\frac{2}{s}, \frac{1}{s}\}$ ($s \in [2,\sqrt{D}]$ for $W\in\mathbb{R}^{D\times d}$), followed by normalization.  Appendix~\ref{app;collinearsparse} discusses why this formulation preserves full rank. Table~\ref{tab:ablasparse} shows experimental results using these sparse bases in \method{}. We explore sparsity ratios $s \in \{2, 6, \sqrt{D}, 100, 200\}$, achieving sparsity levels from $66$ to $99\%$. Consistent with~\citet{2006_ACM_verysparse}, the recommended sparsity levels ($\sqrt{D}$) yield performance comparable to dense matrices, theoretically reducing memory and compute. However, higher sparsity can degrade accuracy, suggesting potential for optimized \method{} variants using compute-optimized sparse random bases.

\subsection{Summary of differences with related random bases algorithms\label{sec:diffvera}}
Prior work like VeRA~\citep{2024_ICLR_VeRA} and NoLA~\citep{2024_ICLR_NoLA} utilizes random bases for parameter-efficient fine-tuning.  However, unlike VeRA and NoLA which approximate a low-rank LoRA update, \method{} aims to approximate the full-rank weight update. It could be argued that VeRA approximates only the first block in a decomposition of $W$, whereas \method{} approximates all blocks. Thus, while VeRA and NoLA improve parameter-efficiency while maintaining low-rank updates, \method{} addresses cases requiring full-rank updates. Furthermore, Equation~\eqref{eq:W} evidences the flexibility in \method{}'s parameter count, ranging from VeRA's parameter efficiency ($r=\text{rank}(W)$) to full fine-tuning parameters ($r=1$) while maintaining full-rank.

\subsection{Limitations}
Despite \method{}'s effectiveness, we identify three key limitations for future research.

First, \method{} introduces computational overhead in weight update calculations, increasing training time for larger models (Appendix~\ref{app:traintime}). We however evidence room for improvement using ternary sparse bases in Section~\ref{sec:randomprojtheory}. Future work should explore matmul-free matrix combinations using these ternary sparse bases. Efficient implementations could replace costly matrix products with simple aggregations, eliminating floating-point arithmetic~\citep{2006_ACM_verysparse}, and accelerating \method{} training time pending the development of optimized CUDA kernels~\citep{2024_arxiv_matmulfree}.

Second, exploring non-random, optimal bases $B_i$ and $A$ could improve convergence and efficiency by further reducing $\epsilon$ in equation~\eqref{eqn;frob_est_assump1}.  Discovering such bases, potentially through experiments or decomposition of pre-trained weights~\citep{2024_loraxs, 2024_NeurIPS_Pissa}, is a promising research direction to enhance \method{}.

Third, hybrid approaches combining LoRA and \method{} warrant investigation. LoRA could estimate the dominant SVD components of $W$, while \method{} captures the remaining spectral information efficiently.  Despite challenges in harmonizing training objectives, a starting point would use \method{} to refine a LoRA when convergence is insufficient. Addressing these limitations will further improve \method{}'s potential for efficient full-rank fine-tuning.

\section{Conclusion}
This paper introduces \method{}, a method achieving parameter efficiency and low memory cost while enabling full rank model updates. Our findings underscore the critical importance of full-rank updates when fine-tuning pre-trained architectures and we observe that our approach surpasses LoRA's performance for an equal parameter count, highlighting the value of full-rank updates in large model fine-tuning. Through extensive experiments across diverse tasks we demonstrated the efficacy of our method. While \method{} incurs additional computational overhead due to random basis multiplications, memory consumption remains contained and we provide venues for reducing this compute in practice. As a results, \method{} offers a viable alternative to LoRA for fine-tuning large pre-trained models on consumer-grade hardware. Our results have significant implications for efficient and effective model adaptation, prompting for future research in scalable and versatile full-rank fine-tuning techniques.

\section*{Acknowledgments}
This research is funded in part by the Australian Government through the Australian Research Council (Project DP240103278), and the Centre of Augmented Reasoning at the Australian Institute for Machine Learning, established by a grant from the Department of Education. This work is also supported by supercomputing resources provided by the Phoenix HPC service at the University of Adelaide. 

\bibliography{iclr2025_conference}
\bibliographystyle{iclr2025_conference}

\newpage
\appendix
\section{3D visualizations of CLIP's loss landscape\label{app:losslandscape}}
\begin{figure}[t]
    \centering
    \begin{subfigure}[b]{0.30\textwidth}
    \includegraphics[width=\textwidth]{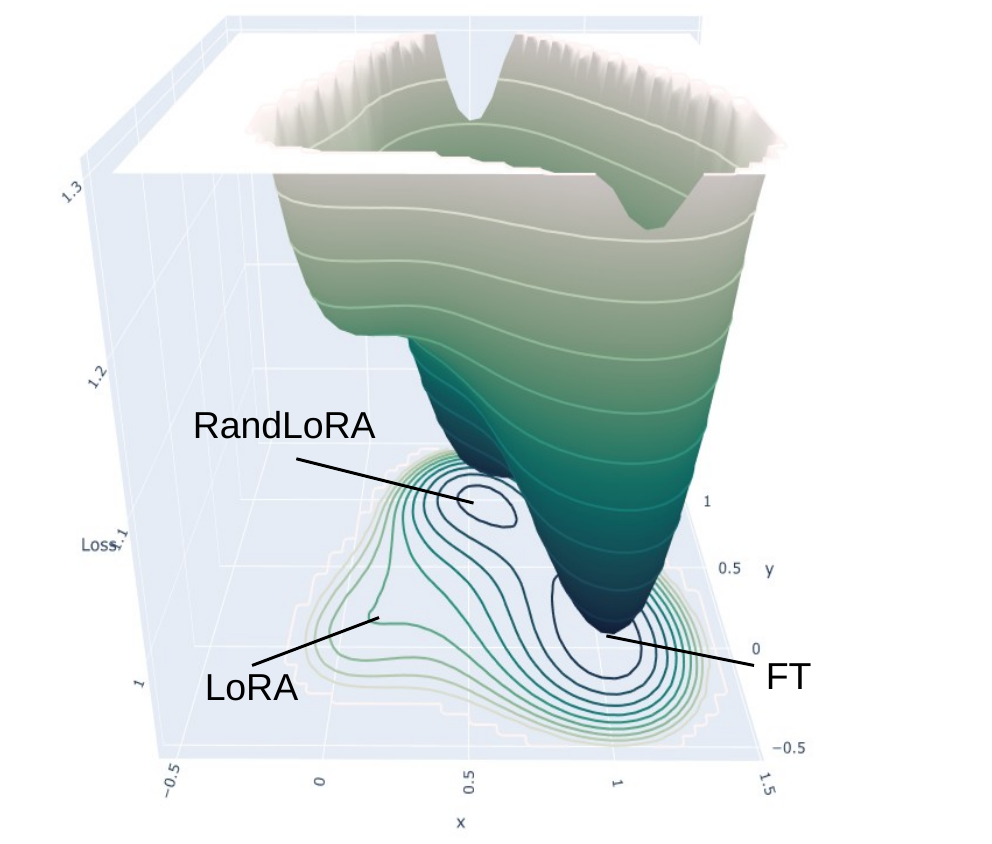}
    \caption{CIFAR-100}
    \label{fig:cifar3d}
    \end{subfigure}
    \quad
    \begin{subfigure}[b]{0.30\textwidth}
    \includegraphics[width=\textwidth]{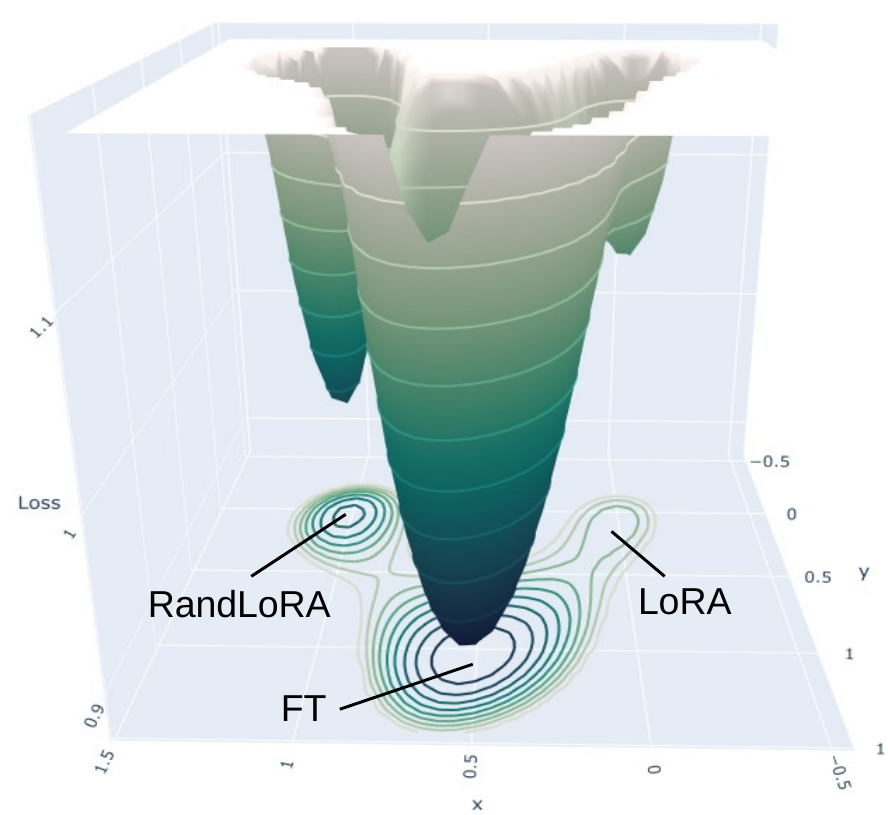}
    \caption{Food-101}
    \label{fig:food3d}
    \end{subfigure}
    \quad
    \begin{subfigure}[b]{0.30\textwidth}
    \includegraphics[width=\textwidth]{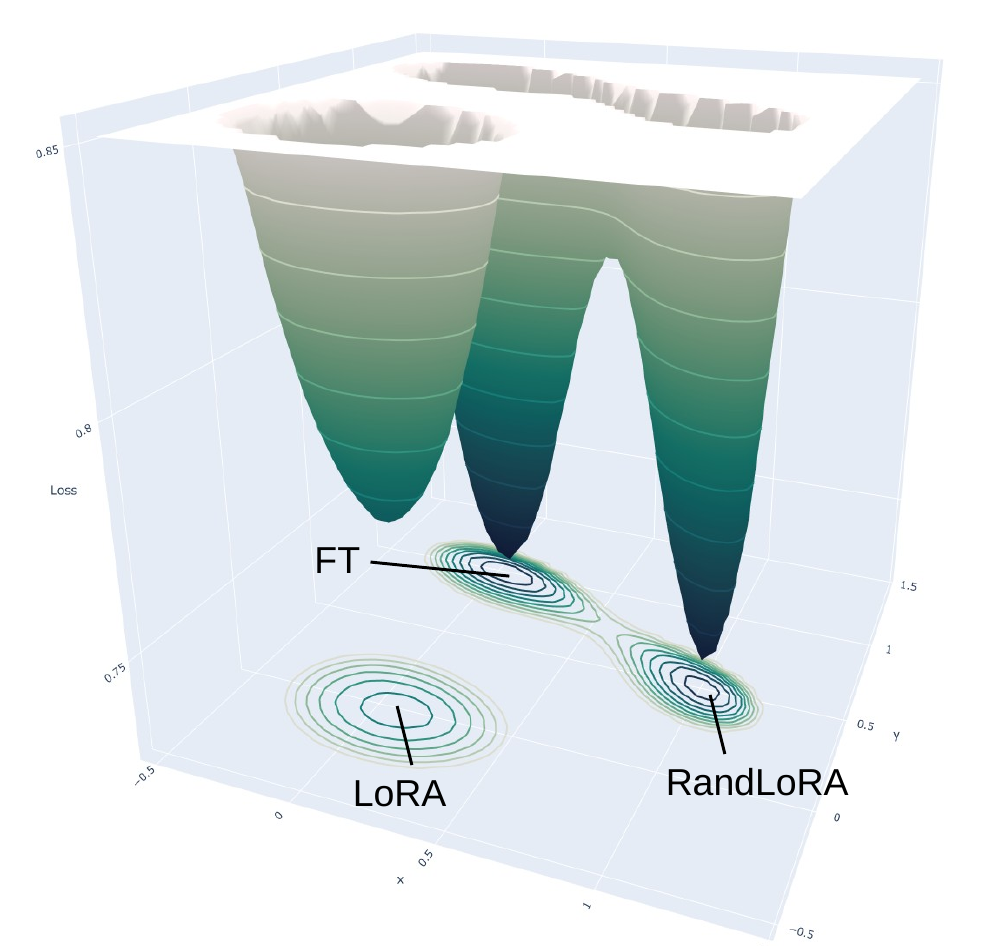}
    \caption{UCF-101}
    \label{fig:ucf3d}
    \end{subfigure}
    \caption{Mode connectivity in the loss landscape when tuning CLIP for image classification. Interactive 3D figures are available in the supplementary material}
    \label{fig:losslandscape3d}
\end{figure}

We propose here further visualizations of the mode connectivity between LoRA, \method{} and standard fine-tuning. To compute the loss value between the minimas reached by LoRA, \method{} and fine-tuning, define a 2D plane using 3 equidistant points representing LoRA, standard fine-tuning and \method{} and we then solve for interpolation coefficients $\alpha_{1..3}$ so that their sum equals 1. The weights of the model we evaluate is then $W_0 + \alpha_1 \text{LoRA} + \alpha_2 \text{FT} + \alpha_3 \text{\method}$. The loss is evaluated on a fixed $5\%$ subset of the training set. Since the process of evaluating the loss at all coordinates on the plane is time consuming, we only perform this study for the CLIP-ViT-B/32 architecture where \method{} is especially successful. In all visualizations, the number of trainable parameters for LoRA and \method{} are the same. We clamp loss values $20\%$ above the shallowest minima to improve visualization.
3D representation as well as the associated 2D elevation projection is provided in Figure~\ref{fig:losslandscape3d}. The interactive 3D figures are provided in the HTML format in the supplementary material.

\section{Additional results\label{app;resultsglue}}
Kronecker\blue{We report here further results on the General Language Understanding Evaluation (GLUE)~\citep{2019_ICLR_glue} and End-to-end (E2E)~\citep{2017_ACL_e2e} generation benchmarks. While GLUE is a text classification task, E2E is a natural language generation task. We also report results comparing \method{} and LoRA with a prompt tuning baseline~\citep{2024_CVPR_DEPT} for classification using CLIP's vision backbone as in section~\ref{sec:expvision} in appendix~\ref{app;deptresults}}

\subsection{GLUE results}
\blue{We report results for \method{} and compare with LoRA and VeRA on the SST-2, MRPC, COLA, QNLI, RTE and STS-N tasks. We report Matthew’s correlation for CoLA, Pearson correlation for STS-B, and accuracy for the remaining tasks. We report results using the RoBERTa network~\cite{2019_arxiv_roberta} in the base and large configurations and perform 5 runs to report average performance and one standard deviation. Results are displayed in Table~\ref{tab:glueresults}. We find that for the smaller RoBERTa-base architecture (125M parameters), all algorithms reach the same performance. For the larger RoBERTa-large variant (355M parameters), a larger gap is observed where \method{} improves over the performance of VeRA and LoRA. These findings are in line with the experiments in the main body of the paper where we find that \method{} provided larger improvements for larger models in Figure~\ref{fig:results-clip}.}

\begin{table}[htbp]
\centering
\small
\setlength{\tabcolsep}{2pt}
\caption{\blue{Results on GLUE datasets with the RoBERTa-base and RoBERTa-large models.}}
\label{tab:glueresults}
\begin{tabular}{lccccccccc}
\toprule
\multicolumn{9}{c}{RoBERTa-base} \\
\midrule
Method & Params & SST-2 & MRPC & COLA & QNLI & RTE & STS-N & Average \\
\midrule
VeRA-1024 & 0.26M & 91.9 $\pm$ 0.4 	& 88.4 $\pm$ 1.2 & 59.9 $\pm$ 2.2 & 90.5 $\pm$ 0.4 & 74.9 $\pm$ 1.5 & 90.4 $\pm$ 0.2 &	82.7 $\pm$ 0.3 \\
LoRA-4 & 0.7M & 94.4 $\pm$ 0.5 & 87.3 $\pm$ 0.2 & 58.4 $\pm$ 0.8 & 92.7 $\pm$ 0.2 & 71.5 $\pm$ 1.2 & 90.5 $\pm$ 0.1 & 82.4 $\pm$ 0.3 \\
RandLoRA-64 & 0.7M & 92.2 $\pm$ 0.3 & 88.0 $\pm$ 1.5 & 59.4 $\pm$ 2.1 & 91.3 $\pm$ 0.4 & 74.7 $\pm$ 1.9 & 90.3 $\pm$ 0.2 & 82.6 $\pm$ 0.5 \\
\midrule
\multicolumn{9}{c}{RoBERTa-large} \\
\midrule
VeRA-256 & 0.26M & 95.8 $\pm$ 0.3 & 89.3 $\pm$ 1.2 & 65.3 $\pm$ 1.1 & 94.1 $\pm$ 0.3 & 81.6 $\pm$ 0.8 & 91.8 $\pm$ 0.1 & 86.3 $\pm$ 0.3 \\
LoRA-4 & 1.8M & 95.5 $\pm$ 0.2 & 87.2 $\pm$ 0.7 & 64.7 $\pm$ 1.2 & 94.5 $\pm$ 0.1 & 83.6 $\pm$ 0.4 & 91.8 $\pm$ 0.1 & 86.2 $\pm$ 0.3 \\
RandLoRA-100 & 1.8M & 95.5 $\pm$ 0.3 & 90.1 $\pm$ 0.4 & 67.4 $\pm$ 0.3 & 94.1 $\pm$ 0.3 & 84.5 $\pm$ 0.3 & 91.4 $\pm$ 0.6 & 87.2 $\pm$ 0.1 \\
\bottomrule
\end{tabular}
\end{table}

\subsection{E2E results}
\blue{We train \method{} and LoRA on the E2E dataset using the GPT-2 medium architecture~\citep{2019_blog_gpt2} (355M parameters).
}

\subsection{Comparison with prompt-tuning~\label{app;deptresults}}
\blue{
Prompt tuning is a popular alternative for PEFT where learnable tokens are appended to human-designed prompts and optimized on to improve accuracy. We choose to report the Maple~\cite{2023_CVPR_maple} + DePT~\cite{2024_CVPR_DEPT} state-of-the-art configuration as it is shown in~\citet{2024_CVPR_DEPT} to be a highly competitive configuration for image classification. Table~\ref{tab:resultsdept} reports the results for 4 and 16 shots over the 11 datasets used in~\citet{2024_CVPR_DEPT}.  We train on ViT-B/32 with all algorithms training approximately 3M parameters. We report that although competitive for low shots, prompt tuning struggles to keep up in the 16-shot setting. We note in particular that prompt tuning struggles on datasets that require more adaptation (e.g. FGVCAircraft) whereas LoRA and \method{} in particular manage to more largely improve results. We additionally report that Maple + DePT requires a much longer training time and VRAM usage. For example, 16-shots on ImageNet requires 3.5h and 18GB of VRAM for Maple + DePT while it requires 2 minutes and 4.5GB of VRAM for RandLoRA. 
Because prompt tuning is largely orthogonal to LoRA-type weight updates we suggest that future research should study how to combine these approaches together.
}
\begin{table}[htbp]
\centering
\small
\setlength{\tabcolsep}{2pt}
\caption{\blue{Comparison of LoRA and \method{} with a state-of-the-art prompt tuning algorithm. CLIP ViT-B/32.\label{tab:resultsdept}}}
\begin{tabular}{llccccccccccccc}
\toprule
Shots & Method & \rotatebox[origin=c]{90}{ImageNet} & \rotatebox[origin=c]{90}{Caltech101} & \rotatebox[origin=c]{90}{OxfordIIITPet} & \rotatebox[origin=c]{90}{Cars} & \rotatebox[origin=c]{90}{Flowers102} & \rotatebox[origin=c]{90}{Food101} & \rotatebox[origin=c]{90}{FGVCAircraft} & \rotatebox[origin=c]{90}{SUN397} & \rotatebox[origin=c]{90}{DTD} & \rotatebox[origin=c]{90}{EuroSAT} & \rotatebox[origin=c]{90}{UCF101} & \rotatebox[origin=c]{90}{Average} \\
\midrule
\multirow{3}{*}{4} & LoRA-16 & \textbf{64.9} & 92.0 & 88.2 & 63.9 & 87.9 & \textbf{82.6} & 30.3 & 68.2 & 61.1 & \textbf{89.4} & 74.7 & 73.0 \\
& RandLoRA-10 & 63.9 & 91.7 & 86.4 & 67.0 & 89.9 & 80.8 & \textbf{34.0} & 69.7 & \textbf{62.4} & 84.4 & 74.9 & 73.2 \\
& Maple + DePT & 62.1 & \textbf{95.0} & \textbf{89.5} & \textbf{68.7} & \textbf{90.5} & 79.6 & 28.3 & \textbf{70.2} & 61.7 & 81.4 & \textbf{76.6} & 73.1 \\
\midrule
\multirow{3}{*}{16} & LoRA-16 & 65.8 & 91.7 & 89.5 & \textbf{80.1} & 94.9 & 81.8 & 42.5 & 73.5 & 72.0 & 91.2 & 81.5 & 78.6 \\
& RandLoRA-10 & 66.3 & 95.6 & \textbf{91.1} & 77.4 & 94.5 & \textbf{84.0} & \textbf{45.0} & 73.7 & \textbf{72.5} & \textbf{94.1} & 81.7 & \textbf{79.6} \\
& Maple + DePT & \textbf{67.7} & \textbf{96.0} & 90.5 & 79.1 & \textbf{96.3} & 81.7 & 36.9 & \textbf{74.5} & 70.3 & 90.3 & \textbf{82.1} & 78.7 \\
\bottomrule
\end{tabular}
\end{table}

\subsection{Commonsense reasoning results for DoRA}
We compare \method{} with DoRA~\citep{2024_ICML_DoRA} for tuning LLama3 in Table~\ref{tab:llmavgresultsdora}. We find that \method{} outperforms both DoRA and LoRA for larger parameter budgets (rank 32), while DoRA and LoRA are competitive at "Efficient" budgets (rank 16).

\begin{table}[t]
\centering
\caption{Further comparison with DoRA related methods on LLama3-8b. Results averaged over 8 commonsense reasoning tasks. We bold the best accuracy.}
\setlength{\tabcolsep}{4pt}
\small
\begin{tabular}{lcccc}
\toprule
\multirow{2}{*}{Method} & \multicolumn{2}{c}{Efficient} &  \multicolumn{2}{c}{Performant} \\
\cmidrule(lr){2-3} \cmidrule(lr){4-5}
 & 15k  & 170k & 15k  & 170k \\
\midrule
LoRA & 82.7 & 84.4 & \textbf{83.1} & 85.2 \\
DoRA & \textbf{82.8} & 84.3 & 82.5 & 85.2 \\
RandLoRA & 81.0 & \textbf{84.6} & 81.3 & \textbf{85.6} \\
\bottomrule
\end{tabular}
\label{tab:llmavgresultsdora}
\end{table}

\section{Implementation details~\label{sec:appimplem}}

\subsection{Classification datasets\label{app:classifdatasets}}
We fine-tune vision architectures on $22$ vision datasets ($21$ for pure vision backbones where ImageNet is removed for brevity). We train for $10$ epochs on the few-shot experiments and increase the number of epochs according to dataset constraints for $50\%$ and $100\%$ fine-tuning. Table~\ref{tab:datasets} reports details of the $22$ datasets we use as well as the number of epochs used as in~\citep{2024_NeurIPS_atlas}.
\begin{table}[h!]
    \captionsetup{font=footnotesize}
    \caption{Vision datasets used for the image classification experiments}
    \label{tab:datasets}
    \centering
    \setlength{\tabcolsep}{2pt} 
    \scriptsize
    \begin{tabularx}{\linewidth}{@{\extracolsep{\fill}}clccccc}
        \toprule
        \# & Datasets & Classes & \multicolumn{3}{c}{Splits} & Epochs \\
        \cline{4-6} \\ [-5pt]
        & & & \textit{train} & \textit{val} & \textit{test} &\\
        \midrule
        (1) & Cars & 196 & 7,330 & 814 & 8,041 & 35 \\
        (2) & DTD & 47 & 3,384 & 376 & 1,880 & 76 \\
        (3) & EuroSAT & 10 & 21,600 & 2,700 & 2,700 & 12 \\
        (4) & GTSRB & 43 & 23,976 & 2,664 & 12,630 & 11 \\
        (5) & MNIST & 10 & 55,000 & 5,000 & 10,000 & 5 \\
        (6) & RESISC45 & 45 & 17,010 & 1,890 & 6,300 & 15 \\
        (7) & SUN397 & 397 & 17,865 & 1,985 & 19,850 & 14 \\
        (8) & SVHN & 10 & 68,257 & 5,000 & 26,032 & 4 \\
        (9) & CIFAR10 & 10 & 45,000 & 5,000 & 10,000 & 5 \\
        (10) & CIFAR100 & 100 & 45,000 & 5,000 & 10,000 & 6 \\
        (11) & ImageNet & 1,000 & 1,276,167 & 5,000 & 50,000 & 10 \\
        (12) & STL10 & 10 & 4,500 & 500 & 8,000 & 4 \\
        (13) & Food101 & 101 & 70,750 & 5,000 & 25,250 & 15 \\
        (14) & Caltech101  & 101 & 6,941 & 694 & 1,736 & 10 \\
        (15) & Caltech256 & 257 & 22,037 & 2,448 & 6,122 & 8\\
        (16) & FGVCAircraft & 100 & 3,334 & 3,333 & 3,333 & 60 \\
        (17) & Flowers102 & 102 & 1,020 & 1,020 & 6,149 & 40 \\
        (18) & OxfordIIITPet & 37 & 3,312 & 368 & 3,669 & 5 \\
        (19) & CUB200 & 200 & 5,395 & 599 & 5,794 & 20 \\
        (20) & PascalVOC & 20 & 7,844 & 7,818 & 14,976 & 10 \\
        (21) & Country211 & 211 & 31,650 & 10,550 & 21,100 & 15 \\
        (22) & UCF101  & 101 & 7,639 & 1,898 & 3,783 & 20 \\
        \bottomrule
    \end{tabularx}
    
\end{table}

\subsection{CLIP}
We utilize the pytorch AdamW optimizer with weight decay $0.1$ and a cosine decaying learning rate schedule. To accommodate the full batch size on a single A100 GPU for the ViT-L/14 and ViT-H/14 CLIP architectures, we accumulate 2 batches of 64. This is excepted for the standard fine-tuning of the ViT-H/14 for standard fine-tuning where we need to accumulate 4 batches of 32 due to increasing memory costs. We acquire the pre-trained weights from the openclip repository~\citep{2023_CVPR_openclip} where the use the "openai" weights from ViT-B/32 and ViT-L/14 and the "laion2b\_s32b\_b79k" weights for ViT-H/14.

\subsection{Pure vision backbones}
For pure vision backbones, we use the same configuration as vision and language fine-tuning of CLIP except that we increase the learning rate to $10^-2$ for LoRA and \method{}. We train \method{}-6 for ViT-B/32 and \method{}-8 for Dinov2's ViT-B/14 and CLIP's ViT-L/14.

\begin{table}[ht]
\centering
\caption{Hyper-parameters for different algorithms. Multiple values for hyperparameters denote variances accross the ViT-B/32, ViT-L/14 and ViT-H/14 architectures respectively.}
\label{tab:hyperparameters}
\begin{tabular}{lcccccc}
\toprule
{Algorithm} & {FT} & {LoRA} & {NoLA} & {VeRA} & {\method{}} \\ 
\midrule
Batch size & 128/64/32 & \multicolumn{4}{c}{128/64/64}\\
Learning Rate (LR) & 1e-5 & 1e-3 & 1e-3 & 1e-2 & 1e-3 \\
Scaling coefficient & 1 & $\frac{1}{r}$ & $\frac{1}{r}$ & $\frac{1}{r}$ & $\frac{10}{r}$ \\
Basis rank (r) & -- & 32 & 1 & 256/256/1024 & 6/8/10 \\
Number of basis ($n$) & -- & -- & 1024 & 1 & 128 \\
\bottomrule
\end{tabular}
\end{table}

\begin{table}[ht]
\centering
\caption{LLM fine-tuning hyper-parameters for different algorithms. Multiple values for hyper-parameters denote variances accross the Qwen2 -0.5b, Phi3-8b and LLama3-8b architectures respectively.}
\label{tab:hyperparametersllm}
\begin{tabular}{lcccccc}
\toprule
{Algorithm} & {LoRA} & {NoLA} & {VeRA} & {\method{}} \\ 
\midrule
Batch size & \multicolumn{4}{c}{16/8/4}\\
Learning Rate (LR) & \multicolumn{4}{c}{$10^{-4}$} \\
Scaling coefficient & 2 & $\frac{2}{\sqrt{n}}$ & 2 & $\frac{2}{\sqrt{n}}$ \\
Basis rank (r) & 32 & 1 & 256/1024/1024 & 6/10/15 \\
Number of basis ($n$) & -- & 1024 & 1 & 149/153/136 \\
\bottomrule
\end{tabular}
\end{table}

\subsection{Commonsense reasoning\label{app:commonsensedatasets}}
Our evaluation protocol assesses the model's versatility and reasoning capabilities across eight diverse datasets: BoolQ~\citep{2019_arXiv_boolq} (yes/no question answering), PIQA~\citep{2020_AAAI_piqa} (physics commonsense questions), SIQA~\citep{2019_arXiv_siqa} (social implications reasoning), HellaSwag~\citep{2019_arXiv_hellaswag} (multi-choice scenario completion), WinoGrande~\citep{2021_arXiv_winogrande} (binary sentence completion), ARC-c and ARC-e~\citep{2018_arXiv_ARC} (challenging and easy science questions at a grade-school level), and OBQA~\citep{2018_arXiv_OBQA} (multi-step reasoning). These datasets collectively pose a wide range of challenges, from natural language understanding and commonsense reasoning to physical and social inference. For further details on these datasets, we refer readers to the survey by Hu et al.~\citep{2023_arXiv_llmadapters}.
We train using the hugginface\footnote{https://huggingface.co} transformers library and follow the implementation\footnote{\url{https://github.com/NVlabs/DoRA/tree/main/commonsense_reasoning}} of Liu et al~\citep{2024_ICML_DoRA}. We train for 3 epochs using a learning rate of $1\times 10^{-4}$ and a base scaling coefficient of $2$ for the weight update. To prevent overfitting, we add a dropout layer in each of the adapter's layers with a dropout probability of $0.05$ and perform early stopping using the same validation set of size $120$, drawn from the training set. We maintain hyper-parameters the same across architectures and algorithms except for the scaling ratio of the weight update for  NoLA and \method{} which we further multiply by $1/\sqrt{n}$ where $n$ is the number of bases to account for the increasing norm of the sum of random matrices.

\subsection{Training time, memory consumption \blue{and random bases}}
\subsubsection{Reducing memory consumption\label{sec:memsupp}}
\paragraph{Basis sharing across layers}
\method{} aims to preserve the memory efficiency and training speed advantages of LoRA. As shown in Section~\ref{sec:theorymethod}, although \method{} trains an amount of parameters comparable to LoRA we still have to store $N$ large random bases for each weight update.
We first note that as observed in previous research, ~\citep{2024_ICLR_NoLA,2024_ICLR_VeRA} random bases can be shared across layers. In practice, we generate one pair of random matrices $B_i \in \mathbb{R}^{N\times D_m \times r}$ and $A_0 \in \mathbb{R}^{\times r \times d_m}$, where $D_m$ and $d_m$ represent the largest $D$ and $d$ across all network layers. During forward and backward passes on a layer of size $D \times d$, we select the first $D$ rows of $B$ and $d$ columns of $A$ to perform the weight update. This strategy stores only the largest $B$ and $A$ matrices, which would have to be fit in memory at some point during training anyways. Note that although we do not study this case, this strategy directly generalizes to having different ranks $r$ across layers as has been proposed in AutoLoRA~\citep{2024_arxiv_autolora} for example. This strategy allows us to avoid increasing memory as network depth increases, meaning that \method{} become more efficient when network depth increases.

\paragraph{Efficient back-propagation with a single random $A$ basis}
We evidence in section~\ref{sec:weightupdate} that the $A_i$ matrices do not need to be $N$ dimensional and that a single $A$ matrix modified by $N$ $\Gamma_i$ is enough to acheive full rank. We can thus optimize the backward pass when computing the gradient of $\Lambda_i$ and $Gamma_i$ so that we only have to store one matrix $A \in \mathbb{R}^{r \times d}$ for the backward pass, further reducing memory consumption. 

\paragraph{Efficient matrix multiplication in the forward pass}
We adopt the notations from Section~\ref{sec:memsupp} to optimize the matrix multiplication of $X\in\mathbb{R}^{B\times D}$ during the forward and backward passes: $XW$. Given the pre-trained weight $W_0\in\mathbb{D\times d}$, LoRA computes $Y = XW_0 + XBA$ where we compute $Y = XW_0 + \sum_{i=1}^N \left(XB_i) ( \Lambda_i A_i \Gamma_i\right)$. These equations suggest \method{} would be $N$ times slower to run than LoRA but in practice, the $XW_0$ operation dominates the matmul time and the $N$ RandLoRA operations are naturally parallelized by the CUDA kernel. In practice we observe a $13\%$ training time increase for the smaller ViT-B-32 models and up to $100\%$ in the worst case for larger models with large weight matrices such as LLama3. 

\blue{
\subsection{Sparse random bases\label{app;collinearsparse}}
We continue here the discussion on the possible collinearity of sparse bases. We remind here that we construct the random bases $B_i$ and $A_i$ by assigning 
$$
\begin{cases}
-1, & \text{with probability } \frac{1}{s} \\
0, & \text{with probability } 1-\frac{2}{s} \\
1, & \text{with probability } \frac{1}{s}
\end{cases}
$$
where $s$ an integer in $[2,\sqrt{D}]$ for $W\in\mathbb{R}^{D\times d}$. Because of the ternary nature of these matrices, there is a non-zero probability that two row are collinear across all random matrices, resulting in non full rank. If we can show that is probability is negligible then the full rank constraint will be preserved in practice.
We compute that the probability of drawing the same size $d$ row twice equates to $p=2\times (\frac{s^2 - 4s + 6}{s^2})^{d}$. Taking the example of the ViT-B/32 architectures with $W\in\mathbb{R}^{768\times 768}$ and for the largest recommended optimal sparsity ($s=\sqrt{768}$) we compute $p=2\times 10^{-49}$. The probability of drawing at least two collinear row over $N$ matrices of is $p_2=(N+D)(N+D-1)p$. In the \method{}-6 configuration for ViT-B, $N=128$ resulting in $p_2=8\times 10^{-44}$ meaning these events are negligible in practice even with a large number of sparse bases and that the full rank constraint is preserved.}

\subsubsection{Training time\label{app:traintime}}
We report in Table~\ref{tab:training-times} the relative training time of \method{} compared to LoRA and standard fine-tuning on a single RTX4090 GPU (A100 for LLama3 and ViT-H/14). Since we do not have ressources to fully fine-tune LLama3, we report LoRA as the memory baseline. In addition to Table~\ref{tab:training-times} we report up to 212\% increase over LoRA-64 training time for the best performing RandLoRA-15 configuration for LLama3-8b. This number should be put in perspective with DoRA leading to a 220\% increase in all configurations for LLama3-8b.

\begin{table}[ht]
\centering
\begin{tabular}{lccccc}
\toprule
{Model} & {Architecture} & {LoRA-32} & \blue{DoRA-32} & {\method{}} & {FT} \\
\midrule
\multirow{3}{*}{CLIP-ViT-B/32} & Training Time & 90 & -- & 113 & 100\% \\
& Memory & 81 & -- & 78 & 100\% \\ 
\hline
\multirow{3}{*}{CLIP-ViT-L/14} & Training Time & 95 & -- & 128 & 100\% \\
& Memory & 72 & -- & 71 & 100\% \\ 
\hline
\multirow{3}{*}{CLIP-ViT-H/14} & Training Time & 96 & -- & 122 & 100\% \\
& Memory & 54 & -- & 51 & 100\% \\ 
\hline
\multirow{3}{*}{LLama3-8B} & Training Time & 100 & \blue{220} & \blue{167} & -- \\
& Memory & 100 & 102 & 102 & -- \\ 
\bottomrule
\end{tabular}
\caption{Comparison of training times for LoRA, \method{}, and FT on vision-language or language architectures.}
\label{tab:training-times}
\end{table}

\section{Mathematical derivations and proofs}\label{app;maths_proof}
\subsection{Theorem~\ref{thm;random_sum_est_main1}\label{app:proof_theo}}

In this section we would like to give the details of the proof of theorem \ref{thm;random_sum_est_main1} from the main paper. In order to do so we will start by proving a few lemmas.

Our method consider decompositions similar to those given in \eqref{eqn;svd_rank1_decomp} and \eqref{eqn;svd_rankk_decomp} that are built from random matrices instead of the left and right singular vectors. A key observation is that such decompositions and their sums will yield high rank matrix approximations. The following two lemmas explains why this is the case.

\begin{lemma}\label{lem;sum_full_rank}
Let $B = [B_1,\ldots, B_n]$ denote a matrix where each 
$B_j \in \mathbb{R}^{D \times r}$ and let $A = [A_1,\ldots, A_n]$ denote a matrix where each $A_j \in \mathbb{R}^{d \times r}$. Assume 
$nr \leq \min(D,d)$ and assume that the columns of $B$ are linearly independent and the columns of $A$ are linearly independent. Define
\begin{equation}
    C = \sum_{j=1}^nB_jA_j^\mathsf{T}
\end{equation}
Then we must have that $rank(C) = nr$.
\end{lemma}
\begin{proof}
We first observe that using the inequality 
$rank(X + Y) \leq rank(X) + rank(Y)$ we get that 
$rank(C) \leq nr$ because each term $B_jA_j^\mathsf{T}$ has rank $r$, since the columns of $A$ and $B$ are linearly independent, and there are $n$ of them.

Then observe that we can rewrite $C$ as 
\begin{equation}
    C = BA^T
\end{equation}
Using Sylvester's rank inequality: If $X \in \mathbb{R}^{D \times l}$ and $Y \in \mathbb{R}^{l \times d}$ then 
\begin{equation}
rank(X) + rank(Y) - l \leq rank(XY)
\end{equation}
we have that
\begin{align}
    rank(C) &= rank(BA^T) \\
    &\geq rank(B) + rank(A^T) - kj \\
    &= 2nr - nr \\
    &= nr
\end{align}
and the proof is complete.
\end{proof}

\begin{lemma}\label{lem;random_draws}
Let $\{X_1,\ldots,X_n\}$ denote $n$ vectors in $\R^N$ where 
$n \leq N$ drawn i.i.d from a Gaussian or uniform distribution. Then with probability $1$ $\{X_1,\ldots,X_n\}$ will be linearly independent. 
\end{lemma}
\begin{proof}
We first note that any measure defined via a Gaussian or Uniform probability distribution is absolutely continuous with respect to the Lebesgue measure. Meaning they have the same sets of measure zero as the Lebesgue measure.

We then prove the case that $\{X_1,\ldots,X_n\}$ are vectors of unit length. Since the vectors were drawn independently, we can first assume we drew $X_1$. The probability that this is the zero vector is $0$ w.r.t the Lebesgue measure on the closed unit ball $B_N(0)$ about the origin in $\R^N$ and hence any other measure absolutely continuous to it. Then draw $X_2$ and note that the probability that $X_2$ lies in 
$span\{X_1\} \cap B_N(0)$ is also $0$ since $span\{X_1\} \cap B_N(0)$ forms a set of $0$ Lebesgue measure in $B_N(0)$. Continuing in this way we find that $\{X_1,\ldots,X_n\}$ will be linearly independent with probability $1$.

For the general case where $\{X_1,\ldots,X_n\}$ are not drawn to have unit length i.e. drawn on the sphere in $\R^N$, we simply note that we can draw each one and then divide by its norm producing one of unit length. Since normalizing by the norm doesn't affect linear independence we get by the above case that  $\{X_1,\ldots,X_n\}$ must be linearly independent with probability $1$.
\end{proof}

Lemmas \ref{lem;sum_full_rank} and \ref{lem;random_draws} show that if we were to i.i.d draw $n$ random vectors $A_1,\ldots,A_n$ in 
$\R^D$ and $n$ vectors $B_1,\ldots, B_n$ using a Gaussian or uniform distribution for $n \leq \min(D,d)$. Then the matrix
$Q = AB^T$ would have rank $n$, where 
$A = [A_1,\ldots, A_n]$ and $B = [B_1,\ldots, B_n]$.

We note that lemma \ref{lem;sum_full_rank} is still true if we were to consider products of the form $B\Lambda A\Gamma$, where $\Lambda$ and $\Gamma$ are diagonal matrices with non-zero diagonal entries.

Using the above two lemmas we can now give a proof of theorem \ref{thm;random_sum_est_main1} from the main paper.

\begin{proof}
The fact that each $B_i\Lambda_iA_i\Gamma_i$ has rank $r$ with probability $1$ follows from lemmas \ref{lem;sum_full_rank} and \ref{lem;random_draws}. In order to estimate the difference
$\lVert W - \sum_{j=1}^nB_j\Lambda_jA_j\Gamma_j\rVert$, we use
\eqref{eqn;svd_rankk_decomp} to write
\begin{equation}
    W = \sum_{j=1}^n U_j\Sigma_jVj^\mathsf{T}.
\end{equation}
We can then estimate
\begin{align}
\lVert W - \sum_{j=1}^nB_j\Lambda_jA_j\Gamma_j\rVert_F &=
\lVert \sum_{j=1}^n U_j\Sigma_jV_j^\mathsf{T} - 
\sum_{j=1}^nB_j\Lambda_jA_j\Gamma_j\rVert_F \\
&= \lVert \sum_{j=1}^n U_j\Sigma_jV^\mathsf{T}_j
- B_j\Lambda_jA_j\Gamma_j\rVert_F \\
&\leq \sum_{j=1}^n\lVert U_j\Sigma_jV^\mathsf{T}_j
- B_j\Lambda_jA_j\Gamma_j\rVert_F \\
&\leq n\cdot\epsilon
\end{align}
where the last inequality follows from the assumption \eqref{eqn;frob_est_assump1}.
\end{proof}

\subsection{LoRA's low bound\label{app:lora_bound}}
We demonstrate here the short derivation leading to the results of equation~\eqref{eq:lora_bound}.
\begin{proof}
By definition, the forbenius norm of a matrix $X\in\mathbb{R}^{n\times n}$, $\vert\vert X\vert\vert_F$ is invariant under left and right multiplications by any orthogonal matrices $P\in\mathbb{R}^{n\times n}$ and $Q\in\mathbb{R}^{n\times n}$, i.e. $\vert\vert X\vert\vert_F = \vert\vert PXQ\vert\vert_F$. Then, given the k-truncated SVD of $M=U\Sigma_k V^T$ with $U, V \in\mathbb{R}^{n\times n}$ and $\Sigma_k \in\mathbb{R}^{n\times n}$ diagonal with elements above the $k$-th being 0, $U$ and $V$ are orthogonal matrices by definition.
We then have the following,
\begin{align}
    \vert\vert X - M \vert\vert_F & = \vert\vert U(X-M)V^T\vert\vert_F \\
    & = \vert\vert \Sigma - \Sigma_k\vert\vert_F \\
    & = \sum_{j=k+1}^r\sigma_j^2
\end{align}
where $\Sigma\in\mathbb{R}^{n\times n}$ is diagonal and contains the $n$ singular values of $X$ by decreasing order and $\sigma_j$ denotes the $j$-th element of $\Sigma$. 

Since by the SVD definition, the best rank-$k$ approximation of $W$ is $M$, given LoRA's rank-$k$ approximation of $W$ by the matrix multiplication $BA$ where $B \in\mathbb{R}^{n\times k}$ and $A \in\mathbb{R}^{k\times n}$ we have
\begin{align}
    \vert\vert X - M \vert\vert_F &\leq \vert\vert X - BA \vert\vert_F\\
     \sum_{j=k+1}^r\sigma_j^2 & \leq \vert\vert X - BA \vert\vert_F.
\end{align}
\end{proof}

\section{Detailed results}
\subsection{Vision language: CLIP\label{app:clipresults}}
We report per dataset accuracies for NoLA, VeRA, LoRA, standard fine-tuning (FT) and \method{} in for the CLIP ViT-B/32 ViT-L/14 and ViT-H/14 architectures on 22 datasets in Tables~\ref{tab:resultsdetailclipvitb32},~\ref{tab:resultsdetailclipvitl14}and~\ref{tab:resultsdetailclipvitlh4} respectively.
\begin{table}[t]
\captionsetup{font=footnotesize}
\caption{Detailed accuracy results per dataset, fine-tuning the vision and language backbones of CLIP-ViT-B/32.
Highest performance and those within a range of $0.1$ in each section are highlighted in bold.\label{tab:resultsdetailclipvitb32}}
    \centering
    \setlength{\tabcolsep}{1pt} 
    \scriptsize
    \begin{tabularx}{\linewidth}{@{\extracolsep{\fill}} ll | cccccccccccccccccccccc c}
    \toprule
     & Method & \rotatebox[origin=c]{90}{Cars} & \rotatebox[origin=c]{90}{DTD} & \rotatebox[origin=c]{90}{EuroSAT} & \rotatebox[origin=c]{90}{GTSRB} & \rotatebox[origin=c]{90}{MNIST} & \rotatebox[origin=c]{90}{RESISC45} & \rotatebox[origin=c]{90}{SUN397} & \rotatebox[origin=c]{90}{SVHN} & \rotatebox[origin=c]{90}{CIFAR10} & \rotatebox[origin=c]{90}{CIFAR100} & \rotatebox[origin=c]{90}{ImageNet} & \rotatebox[origin=c]{90}{STL10} & \rotatebox[origin=c]{90}{Food101} & \rotatebox[origin=c]{90}{Caltech256} & \rotatebox[origin=c]{90}{FGVCAircraft} & \rotatebox[origin=c]{90}{Flowers102} & \rotatebox[origin=c]{90}{OxfordIIITPet} & \rotatebox[origin=c]{90}{CUB200} & \rotatebox[origin=c]{90}{PascalVOC} & \rotatebox[origin=c]{90}{Country211} & \rotatebox[origin=c]{90}{Caltech101} & \rotatebox[origin=c]{90}{UCF101} & \rotatebox[origin=c]{90}{Average}\tabularnewline
    \midrule
 \multirow{5}{*}{1 shot} & NoLA & 51.6 & 44.5 & 72.8 & 54.3 & 76.3 & 64.1 & 53.8 & 31.1 & 81.3 & 62.7 & 49.7 & 90.4 & 61.9 & 76.6 & 19.0 & 62.5 & 69.7 & 41.8 & 69.1 & 5.3 & 84.9 & 61.4 & 58.4 \\
 & VeRA256 & \textbf{60.9} & 47.7 & \textbf{76.8} & 47.4 & 71.7 & 67.4 & \textbf{64.9} & \textbf{47.5} & \textbf{90.4} & \textbf{71.7} & \textbf{63.7} & \textbf{97.4} & \textbf{83.5} & \textbf{83.3} & \textbf{22.1} & 68.5 & \textbf{88.3} & \textbf{54.4} & \textbf{77.6} & \textbf{17.6} & 87.5 & 64.9 & \textbf{66.1} \\
 & LoRA32 & 51.9 & 46.3 & 73.2 & 61.4 & 73.7 & 67.9 & 53.9 & 30.6 & 79.8 & 63.9 & 51.7 & 89.5 & 63.5 & 78.1 & 19.1 & 65.3 & 69.9 & 43.0 & 67.1 & 5.6 & 85.2 & 63.5 & 59.3 \\
 & RandLoRA6 & 53.6 & \textbf{50.3} & 73.1 & 61.4 & \textbf{78.5} & \textbf{72.6} & 59.3 & 29.4 & 80.8 & 67.1 & 57.4 & 92.6 & 69.8 & 81.5 & 21.7 & \textbf{71.3} & 75.0 & 48.5 & 67.6 & 8.5 & \textbf{88.3} & \textbf{67.0} & 62.5 \\
 & FT & 51.4 & 46.8 & 67.3 & \textbf{62.8} & 77.4 & 69.9 & 57.2 & 20.0 & 68.3 & 61.1 & 52.2 & 83.0 & 66.7 & 79.5 & 19.0 & 68.7 & 70.0 & 46.5 & 59.0 & 7.4 & 86.1 & 66.6 & 58.5 \\
\midrule
\multirow{5}{*}{2 shots} & NoLA & 57.1 & 54.3 & 82.8 & 63.6 & 83.2 & 69.7 & 57.9 & 32.2 & 80.3 & 68.5 & 51.0 & 92.2 & 67.2 & 80.4 & 24.3 & 72.7 & 80.8 & 47.3 & 57.9 & 7.4 & 85.2 & 65.4 & 62.8 \\
 & VeRA256 & \textbf{62.1} & 49.5 & 71.0 & 50.5 & 72.2 & 68.1 & \textbf{64.8} & \textbf{50.7} & \textbf{91.7} & \textbf{73.1} & \textbf{63.7} & \textbf{97.5} & \textbf{84.2} & \textbf{84.0} & 22.1 & 69.9 & \textbf{89.2} & 54.8 & \textbf{73.8} & \textbf{17.7} & \textbf{89.2} & 65.0 & 66.6 \\
 & LoRA32 & 53.7 & 56.9 & 82.0 & 62.6 & 82.8 & 71.9 & 60.1 & 36.8 & 84.2 & 71.5 & 52.9 & 94.1 & 73.6 & 82.8 & 22.4 & 73.8 & 84.2 & 48.0 & 61.7 & 9.0 & 87.8 & 67.4 & 64.6 \\
 & RandLoRA6 & 59.5 & \textbf{60.4} & \textbf{83.4} & 73.7 & \textbf{85.2} & \textbf{74.9} & 62.0 & 30.0 & 82.6 & 72.0 & 57.7 & 94.5 & 72.0 & 83.8 & \textbf{28.6} & 80.8 & 83.7 & 54.3 & 62.3 & 9.8 & 89.0 & \textbf{71.7} & \textbf{66.9} \\
 & FT & 58.5 & 57.7 & 82.9 & \textbf{76.7} & 84.8 & 74.4 & 60.3 & 23.0 & 69.4 & 68.3 & 53.9 & 87.3 & 69.1 & 83.0 & 26.2 & \textbf{81.0} & 79.1 & \textbf{55.2} & 53.2 & 9.3 & \textbf{89.2} & \textbf{71.6} & 64.3 \\
\midrule
\multirow{5}{*}{4 shots} & NoLA & 60.1 & 58.1 & 86.9 & 67.7 & 87.5 & 75.0 & 61.0 & 45.3 & 87.2 & 69.3 & 51.4 & 91.3 & 72.3 & 81.2 & 26.0 & 80.7 & 84.1 & 51.6 & 69.0 & 9.3 & 87.3 & 68.0 & 66.8 \\
 & VeRA256 & 61.8 & 49.6 & 79.7 & 52.5 & 73.2 & 69.6 & \textbf{64.9} & \textbf{52.2} & \textbf{92.3} & \textbf{73.9} & \textbf{64.2} & \textbf{97.5} & \textbf{84.9} & 83.8 & 21.9 & 70.4 & \textbf{89.5} & 54.9 & \textbf{75.8} & \textbf{17.8} & 89.4 & 65.6 & 67.5 \\
 & LoRA32 & 57.0 & 60.4 & 86.7 & 59.0 & 86.5 & 73.5 & 62.3 & 46.4 & 87.1 & 71.1 & 52.5 & 93.6 & 76.3 & 83.2 & 24.2 & 77.2 & 84.7 & 50.9 & 69.5 & 11.2 & 88.4 & 67.1 & 66.8 \\
 & RandLoRA6 & 63.1 & \textbf{63.2} & \textbf{87.9} & 77.4 & \textbf{88.2} & \textbf{80.3} & \textbf{65.0} & 47.8 & 87.6 & 72.9 & 55.8 & 93.2 & 74.8 & \textbf{84.1} & 31.1 & 87.8 & 85.0 & 58.8 & 70.3 & 10.7 & \textbf{89.8} & \textbf{75.3} & \textbf{70.4} \\
 & FT & \textbf{65.2} & 60.3 & 85.4 & \textbf{82.5} & 87.0 & 80.1 & 64.1 & 41.1 & 78.9 & 70.8 & 54.0 & 84.3 & 72.0 & 83.2 & \textbf{34.1} & \textbf{89.5} & 80.1 & \textbf{60.1} & 62.5 & 10.0 & 89.6 & 73.8 & 68.6 \\
\midrule
\multirow{5}{*}{16 shots} & NoLA & 66.2 & 66.5 & 92.3 & 73.6 & 91.2 & 81.2 & 64.4 & \textbf{74.9} & 92.1 & 74.3 & 54.0 & 95.0 & 77.3 & 84.0 & 30.4 & 86.0 & 89.6 & 61.1 & 73.5 & 12.0 & 88.2 & 73.7 & 72.8 \\
 & VeRA256 & 62.9 & 51.4 & 82.4 & 53.2 & 75.8 & 70.5 & 66.3 & 57.0 & \textbf{93.3} & 73.9 & \textbf{64.6} & \textbf{97.9} & \textbf{85.2} & 85.6 & 22.3 & 71.6 & \textbf{90.9} & 55.8 & \textbf{76.4} & \textbf{18.1} & 89.2 & 65.7 & 68.6 \\
 & LoRA32 & 69.6 & 64.8 & 87.5 & 61.2 & 91.2 & 79.8 & 65.0 & 71.6 & 93.0 & 75.7 & 54.9 & 95.8 & 77.3 & 85.8 & 33.7 & 83.3 & 89.6 & 64.4 & 75.2 & 12.1 & 88.5 & 76.3 & 72.6 \\
 & RandLoRA6 & 71.9 & \textbf{70.2} & \textbf{94.2} & 81.5 & 94.1 & 84.9 & \textbf{67.6} & 73.7 & 92.0 & \textbf{77.0} & 56.8 & 95.0 & 80.1 & \textbf{86.9} & 35.1 & 91.3 & 89.3 & 68.6 & 75.5 & 12.2 & \textbf{90.9} & \textbf{79.3} & 75.8 \\
 & FT & \textbf{74.0} & 69.8 & 93.2 & \textbf{87.5} & \textbf{94.3} & \textbf{86.7} & 67.2 & 74.1 & 89.8 & 76.3 & 56.2 & 92.7 & 78.6 & \textbf{86.9} & \textbf{39.1} & \textbf{93.2} & 89.0 & \textbf{70.1} & 74.9 & 12.1 & \textbf{90.9} & 78.9 & \textbf{76.2} \\
\midrule
\multirow{5}{*}{50\%} & NoLA & 69.7 & 68.9 & 98.6 & 93.9 & 98.7 & 91.6 & 64.9 & 93.0 & 97.1 & 79.0 & 56.9 & 97.8 & 81.0 & 86.3 & 44.2 & 81.9 & 89.6 & 62.3 & 85.6 & 14.4 & 88.9 & 78.0 & 78.3 \\
 & VeRA256 & 63.7 & 62.4 & 95.5 & 79.2 & 92.8 & 81.1 & 66.3 & 75.6 & 95.2 & 76.3 & \textbf{64.6} & \textbf{97.9} & 85.6 & 87.9 & 25.6 & 72.1 & 88.8 & 56.6 & 85.4 & \textbf{18.1} & 93.3 & 70.7 & 74.3 \\
 & LoRA32 & 71.9 & 71.3 & 98.4 & 94.7 & 98.8 & 93.0 & 65.6 & 93.7 & 97.4 & 81.5 & 59.5 & 97.7 & 85.4 & 88.1 & 45.3 & 85.8 & 89.2 & 65.2 & 86.5 & 14.1 & 88.5 & 80.2 & 79.6 \\
 & RandLoRA6 & \textbf{78.0} & \textbf{73.6} & 98.5 & 95.5 & 99.0 & 94.0 & \textbf{67.4} & 94.6 & \textbf{97.7} & 84.4 & 62.4 & 97.9 & \textbf{87.6} & \textbf{89.5} & 56.3 & 88.5 & 90.0 & \textbf{70.3} & 86.5 & 14.6 & \textbf{95.3} & \textbf{82.5} & \textbf{82.0} \\
 & FT & \textbf{78.0} & 72.4 & \textbf{98.7} & \textbf{96.2} & \textbf{99.1} & \textbf{94.5} & 67.0 & \textbf{95.0} & 97.6 & \textbf{84.8} & 62.1 & \textbf{98.0} & 86.6 & 89.2 & \textbf{57.4} & \textbf{89.1} & \textbf{91.1} & 69.0 & \textbf{87.2} & 14.6 & 94.9 & 81.8 & \textbf{82.0} \\
\midrule
\multirow{5}{*}{100\%} & NoLA & 73.6 & 73.5 & 98.8 & 95.2 & 99.0 & 93.3 & 66.4 & 94.2 & 97.6 & 80.3 & 57.5 & 98.1 & 82.0 & 87.5 & 51.1 & 89.1 & 90.8 & 67.1 & 86.5 & 15.9 & 90.1 & 78.4 & 80.3 \\
 & VeRA256 & 63.7 & 62.5 & 95.2 & 79.5 & 92.2 & 80.6 & 66.3 & 75.4 & 95.2 & 76.2 & \textbf{64.6} & 98.1 & 85.6 & 87.8 & 25.4 & 77.3 & 90.6 & 56.8 & 85.9 & \textbf{18.1} & 93.8 & 70.3 & 74.6 \\
 & LoRA32 & 77.3 & 76.7 & 98.6 & 95.3 & 99.1 & 94.4 & 67.1 & 95.2 & 97.9 & 83.8 & 60.5 & \textbf{98.4} & 87.8 & 89.2 & 59.5 & 91.4 & 91.1 & 70.7 & 87.7 & 15.9 & 89.6 & 82.0 & 82.2 \\
 & RandLoRA6 & 83.1 & \textbf{78.9} & \textbf{99.0} & 96.1 & \textbf{99.3} & 95.4 & \textbf{69.5} & 95.5 & \textbf{98.1} & \textbf{87.0} & 63.8 & \textbf{98.4} & \textbf{89.4} & \textbf{90.9} & 67.1 & 93.7 & 91.0 & \textbf{75.2} & \textbf{88.0} & 16.8 & 95.6 & \textbf{85.1} & \textbf{84.4} \\
 & FT & \textbf{84.4} & 77.7 & 98.9 & \textbf{96.8} & 99.2 & \textbf{96.0} & 69.0 & \textbf{96.0} & 97.9 & 86.9 & 63.7 & \textbf{98.5} & 88.8 & 90.8 & \textbf{68.1} & \textbf{94.8} & \textbf{91.2} & 74.8 & \textbf{88.0} & 16.3 & \textbf{95.8} & 84.6 & \textbf{84.5} \\
    \bottomrule
    \end{tabularx}
\end{table}

\begin{table}[t]
\captionsetup{font=footnotesize}
\caption{Detailed accuracy results per dataset, fine-tuning the vision and language backbones of CLIP-ViT-L/14.
Highest performance and those within a range of $0.1$ in each section are highlighted in bold.\label{tab:resultsdetailclipvitl14}}
    \centering
    \setlength{\tabcolsep}{1pt} 
    \scriptsize
    \begin{tabularx}{\linewidth}{@{\extracolsep{\fill}} ll | cccccccccccccccccccccc c}
    \toprule
     & Method & \rotatebox[origin=c]{90}{Cars} & \rotatebox[origin=c]{90}{DTD} & \rotatebox[origin=c]{90}{EuroSAT} & \rotatebox[origin=c]{90}{GTSRB} & \rotatebox[origin=c]{90}{MNIST} & \rotatebox[origin=c]{90}{RESISC45} & \rotatebox[origin=c]{90}{SUN397} & \rotatebox[origin=c]{90}{SVHN} & \rotatebox[origin=c]{90}{CIFAR10} & \rotatebox[origin=c]{90}{CIFAR100} & \rotatebox[origin=c]{90}{ImageNet} & \rotatebox[origin=c]{90}{STL10} & \rotatebox[origin=c]{90}{Food101} & \rotatebox[origin=c]{90}{Caltech256} & \rotatebox[origin=c]{90}{FGVCAircraft} & \rotatebox[origin=c]{90}{Flowers102} & \rotatebox[origin=c]{90}{OxfordIIITPet} & \rotatebox[origin=c]{90}{CUB200} & \rotatebox[origin=c]{90}{PascalVOC} & \rotatebox[origin=c]{90}{Country211} & \rotatebox[origin=c]{90}{Caltech101} & \rotatebox[origin=c]{90}{UCF101} & \rotatebox[origin=c]{90}{Average}\tabularnewline
    \midrule
\multirow{5}{*}{1 shot} & NoLA & 72.7 & 61.1 & 81.5 & 76.4 & 89.3 & 78.6 & 67.3 & \textbf{76.1} & 94.0 & 77.9 & 70.3 & 98.8 & 87.5 & 88.3 & 41.1 & 85.0 & 90.4 & 63.4 & 71.7 & 18.0 & 90.4 & 76.6 & 75.3 \\
 & VeRA256 & \textbf{78.5} & 55.6 & 75.3 & 55.0 & 88.8 & 73.2 & 68.8 & 67.8 & \textbf{96.6} & 80.5 & 75.5 & \textbf{99.4} & \textbf{93.2} & 88.9 & 34.3 & 80.6 & \textbf{93.8} & 64.2 & \textbf{78.8} & \textbf{32.0} & 86.8 & 73.6 & 74.6 \\
 & LoRA32 & 74.9 & 62.3 & 81.0 & 76.5 & 91.7 & 79.5 & 68.3 & 74.7 & 92.8 & 78.9 & 71.4 & 98.6 & 87.9 & 89.4 & \textbf{44.0} & 89.5 & 88.7 & 66.3 & 68.5 & 19.3 & 90.3 & 77.7 & 76.0 \\
 & \method{}10 & 76.8 & \textbf{63.1} & \textbf{83.5} & 72.5 & \textbf{92.7} & 81.6 & \textbf{74.7} & 74.2 & 95.0 & \textbf{83.0} & \textbf{76.2} & 99.3 & 91.6 & \textbf{92.1} & 43.2 & 89.1 & 91.0 & 68.8 & 74.9 & 27.2 & 90.3 & \textbf{82.7} & \textbf{78.3} \\
 & FT & 73.6 & 62.4 & 81.2 & \textbf{78.4} & \textbf{92.8} & \textbf{83.8} & 71.5 & 68.3 & 91.0 & 81.3 & 73.2 & 98.6 & 88.4 & 91.6 & 41.8 & \textbf{90.5} & 88.7 & \textbf{68.9} & 66.1 & 23.0 & \textbf{90.7} & 82.5 & 76.7 \\
\midrule
\multirow{5}{*}{2 shots} & NoLA & 74.0 & 66.7 & 81.1 & 81.2 & 93.2 & 82.4 & 68.0 & 78.3 & 93.3 & 80.8 & 66.3 & 98.2 & 88.0 & 89.4 & 39.6 & 92.5 & 93.9 & 64.8 & 75.2 & 20.5 & 91.2 & 76.6 & 77.1 \\
 & VeRA256 & 78.1 & 55.8 & 75.3 & 55.7 & 90.0 & 73.5 & 68.6 & 67.0 & \textbf{96.6} & 81.3 & \textbf{75.6} & \textbf{99.4} & \textbf{93.2} & 89.0 & 34.8 & 81.5 & \textbf{94.4} & 64.0 & \textbf{79.2} & \textbf{32.2} & 86.8 & 74.1 & 74.8 \\
 & LoRA32 & 77.3 & 68.1 & 84.7 & 82.7 & \textbf{95.2} & 84.2 & 69.9 & \textbf{78.5} & 92.4 & 81.6 & 68.7 & 97.8 & 88.7 & 90.0 & 46.4 & 94.5 & 91.8 & 69.3 & 72.6 & 21.0 & 91.7 & 79.5 & 78.5 \\
 & \method{}10 & 78.5 & 70.4 & \textbf{85.1} & 80.4 & 94.7 & 85.8 & \textbf{74.9} & 78.2 & 95.9 & \textbf{84.1} & 74.4 & \textbf{99.5} & 91.9 & \textbf{92.5} & 46.1 & 94.5 & 93.9 & 71.5 & 75.8 & 28.1 & 91.7 & 83.6 & \textbf{80.5} \\
 & FT & \textbf{79.6} & \textbf{70.5} & 83.8 & \textbf{84.0} & 94.0 & \textbf{86.5} & 73.2 & 78.1 & 92.5 & 82.7 & 72.1 & 99.2 & 89.2 & 91.6 & \textbf{47.2} & \textbf{96.5} & 93.1 & \textbf{73.5} & 72.7 & 24.1 & \textbf{91.9} & \textbf{84.1} & 80.0 \\
\midrule
\multirow{5}{*}{4 shots} & NoLA & 75.2 & 70.0 & 87.4 & 85.5 & 95.5 & 84.4 & 69.2 & 82.5 & 94.8 & 82.2 & 66.2 & 97.9 & 89.3 & 89.7 & 44.5 & 93.1 & 94.2 & 67.3 & 77.0 & 23.0 & 91.3 & 77.2 & 79.0 \\
 & VeRA256 & 77.9 & 56.7 & 77.8 & 56.0 & 91.3 & 74.1 & 69.8 & 68.0 & 96.9 & 81.4 & \textbf{75.9} & \textbf{99.5} & \textbf{93.2} & 89.1 & 35.1 & 81.1 & 94.6 & 64.2 & \textbf{79.3} & \textbf{32.1} & 86.9 & 74.2 & 75.2 \\
 & LoRA32 & 77.2 & 71.8 & 88.4 & 86.2 & 95.9 & 86.3 & 70.5 & \textbf{84.3} & 95.1 & 82.4 & 68.7 & 97.5 & 90.2 & 90.8 & \textbf{47.4} & 95.5 & 93.7 & 70.6 & 75.8 & 23.2 & 91.8 & 81.4 & 80.2 \\
 & \method{}10 & 79.3 & 73.6 & 89.2 & 85.2 & \textbf{96.4} & 87.8 & \textbf{74.6} & 80.9 & \textbf{97.3} & \textbf{85.1} & 72.6 & 99.3 & 92.4 & 92.4 & 47.1 & 93.7 & \textbf{94.8} & 71.0 & 79.1 & 29.2 & 91.7 & 84.6 & \textbf{81.7} \\
 & FT & \textbf{79.7} & \textbf{74.6} & \textbf{90.0} & \textbf{90.1} & 96.0 & \textbf{88.8} & 73.5 & 82.5 & 94.2 & 84.2 & 71.6 & 98.1 & 89.8 & \textbf{92.7} & 43.3 & \textbf{97.3} & 93.7 & \textbf{76.0} & 78.2 & 25.3 & \textbf{92.3} & \textbf{84.8} & \textbf{81.7} \\
\midrule
\multirow{5}{*}{16 shots} & NoLA & 82.8 & 72.0 & 93.7 & 86.4 & 96.7 & 87.3 & 72.2 & 87.8 & 97.0 & 84.2 & 69.1 & 98.7 & 90.5 & 93.0 & 53.5 & 96.2 & 94.6 & 78.8 & \textbf{83.7} & 23.6 & 90.3 & 82.7 & 82.5 \\
 & VeRA256 & 80.5 & 56.1 & 82.6 & 56.2 & 93.9 & 74.4 & 71.9 & 69.8 & \textbf{97.2} & 83.0 & \textbf{76.3} & \textbf{99.5} & \textbf{93.5} & 90.3 & 38.3 & 82.3 & 94.8 & 68.3 & 80.2 & \textbf{32.8} & 89.1 & 77.2 & 76.7 \\
 & LoRA32 & 85.7 & 74.8 & 94.2 & 88.1 & 97.1 & 88.9 & 73.3 & \textbf{88.7} & 96.9 & 85.8 & 70.9 & 99.0 & 91.2 & 93.2 & 56.7 & 97.5 & 94.2 & 82.6 & 82.1 & 23.6 & 90.8 & 85.5 & 83.7 \\
 & \method{}10 & 86.6 & 76.0 & 94.9 & 87.4 & 97.2 & 89.4 & \textbf{76.5} & 86.4 & 97.0 & 86.5 & 74.5 & 99.2 & 92.3 & \textbf{94.4} & 57.4 & 97.8 & \textbf{95.3} & 83.9 & 82.4 & 25.3 & 91.7 & \textbf{88.5} & 84.6 \\
 & FT & \textbf{87.5} & \textbf{78.4} & \textbf{95.7} & \textbf{91.7} & \textbf{97.7} & \textbf{91.2} & 75.6 & 87.4 & 94.6 & \textbf{87.3} & 73.5 & 98.3 & 91.4 & 94.1 & \textbf{61.1} & \textbf{98.4} & 94.2 & \textbf{85.0} & 82.3 & 25.8 & \textbf{92.9} & 88.2 & \textbf{85.1} \\
\midrule
\multirow{5}{*}{50\%} & NoLA & 84.4 & 78.0 & 98.6 & 96.4 & \textbf{99.3} & 95.2 & 72.7 & 96.3 & 99.1 & 89.2 & 73.0 & \textbf{99.5} & 93.0 & 94.4 & 57.9 & 96.3 & 95.6 & 79.3 & 91.5 & 26.7 & 91.3 & 87.0 & 86.1 \\
 & VeRA256 & 81.7 & 68.8 & 95.8 & 88.5 & 97.0 & 86.8 & 71.8 & 90.5 & 98.1 & 85.0 & 76.2 & \textbf{99.5} & 93.9 & 93.8 & 44.5 & 87.7 & 94.4 & 70.2 & 88.6 & \textbf{32.9} & 94.3 & 80.9 & 82.8 \\
 & LoRA32 & 88.2 & 81.2 & 98.8 & \textbf{96.9} & 99.1 & 96.0 & 74.1 & 96.5 & \textbf{99.2} & 90.3 & 75.4 & \textbf{99.5} & 94.4 & 95.6 & 68.4 & \textbf{97.2} & 94.9 & 83.2 & 91.0 & 25.5 & 94.1 & 88.4 & 87.6 \\
 & \method{}10 & \textbf{89.9} & \textbf{82.3} & 98.8 & 96.8 & \textbf{99.4} & 96.0 & \textbf{76.7} & 96.8 & \textbf{99.2} & \textbf{91.6} & \textbf{78.3} & \textbf{99.5} & \textbf{94.7} & 95.6 & 69.0 & 96.9 & \textbf{95.7} & 83.9 & \textbf{92.1} & 27.5 & \textbf{96.9} & 90.5 & \textbf{88.5} \\
 & FT & 89.7 & 79.0 & \textbf{99.1} & 96.3 & 99.3 & \textbf{96.8} & 76.0 & \textbf{97.0} & \textbf{99.2} & 91.2 & 77.3 & 99.4 & 94.3 & \textbf{95.8} & \textbf{69.6} & \textbf{97.1} & 95.0 & \textbf{84.6} & 91.9 & 26.8 & \textbf{96.9} & \textbf{90.8} & 88.3 \\
\midrule
\multirow{5}{*}{100\%} & NoLA & 87.5 & 82.5 & 99.0 & 96.8 & 99.3 & 96.3 & 75.0 & 96.6 & 99.3 & 90.4 & 73.6 & \textbf{99.7} & 93.8 & 95.2 & 74.0 & 98.5 & 95.1 & 83.2 & 91.5 & 28.5 & 93.9 & 88.0 & 88.1 \\
 & VeRA256 & 81.6 & 67.9 & 96.1 & 88.6 & 97.2 & 85.8 & 71.7 & 90.2 & 98.2 & 85.1 & 77.0 & 99.5 & 93.8 & 93.9 & 44.5 & 93.0 & 94.9 & 70.3 & 89.2 & \textbf{32.8} & 96.4 & 81.5 & 83.1 \\
 & LoRA32 & 89.2 & 83.9 & \textbf{99.2} & \textbf{97.4} & 99.3 & 96.8 & 75.9 & 95.8 & \textbf{99.3} & 91.4 & 76.1 & \textbf{99.7} & 95.2 & 95.8 & 78.6 & 98.4 & 95.2 & 85.3 & 91.7 & 27.9 & 96.1 & 90.2 & 89.0 \\
 & \method{}10 & \textbf{90.8} & \textbf{84.6} & 99.0 & 96.6 & \textbf{99.5} & 96.9 & \textbf{77.8} & 97.0 & \textbf{99.4} & \textbf{92.8} & \textbf{79.0} & \textbf{99.7} & \textbf{95.4} & \textbf{96.5} & 79.6 & 98.9 & \textbf{95.4} & \textbf{87.1} & 92.5 & 30.4 & 96.8 & \textbf{93.1} & \textbf{90.0} \\
 & FT & 90.4 & 84.4 & \textbf{99.1} & 97.1 & 99.3 & \textbf{97.2} & 77.2 & \textbf{97.3} & 99.2 & 92.4 & 78.1 & \textbf{99.6} & 94.9 & 96.2 & \textbf{81.5} & \textbf{99.1} & 94.8 & 86.9 & \textbf{92.6} & 29.3 & \textbf{97.0} & 92.6 & 89.8 \\
    \bottomrule
    \end{tabularx}
\end{table}

\begin{table}[t]
\captionsetup{font=footnotesize}
\caption{Detailed accuracy results per dataset, fine-tuning the vision and language backbones of CLIP-ViT-H/14.
Highest performance and those within a range of $0.1$ in each section are highlighted in bold.\label{tab:resultsdetailclipvitlh4}}
    \centering
    \setlength{\tabcolsep}{1pt} 
    \scriptsize
    \begin{tabularx}{\linewidth}{@{\extracolsep{\fill}} ll | cccccccccccccccccccccc c}
    \toprule
     & Method & \rotatebox[origin=c]{90}{Cars} & \rotatebox[origin=c]{90}{DTD} & \rotatebox[origin=c]{90}{EuroSAT} & \rotatebox[origin=c]{90}{GTSRB} & \rotatebox[origin=c]{90}{MNIST} & \rotatebox[origin=c]{90}{RESISC45} & \rotatebox[origin=c]{90}{SUN397} & \rotatebox[origin=c]{90}{SVHN} & \rotatebox[origin=c]{90}{CIFAR10} & \rotatebox[origin=c]{90}{CIFAR100} & \rotatebox[origin=c]{90}{ImageNet} & \rotatebox[origin=c]{90}{STL10} & \rotatebox[origin=c]{90}{Food101} & \rotatebox[origin=c]{90}{Caltech256} & \rotatebox[origin=c]{90}{FGVCAircraft} & \rotatebox[origin=c]{90}{Flowers102} & \rotatebox[origin=c]{90}{OxfordIIITPet} & \rotatebox[origin=c]{90}{CUB200} & \rotatebox[origin=c]{90}{PascalVOC} & \rotatebox[origin=c]{90}{Country211} & \rotatebox[origin=c]{90}{Caltech101} & \rotatebox[origin=c]{90}{UCF101} & \rotatebox[origin=c]{90}{Average}\tabularnewline
    \midrule
    \multirow{5}{*}{1 shot} & NoLA & 92.0 & \textbf{71.8} & 80.9 & 78.7 & \textbf{90.1} & 82.2 & 70.6 & \textbf{60.2} & 95.1 & 82.7 & 73.1 & 98.0 & 88.0 & 90.3 & 46.4 & 91.5 & 91.2 & 76.2 & 69.7 & 17.5 & 91.5 & 78.1 & 78.0 \\
 & VeRA1024 & \textbf{93.8} & 69.4 & 73.8 & 65.1 & 90.0 & 73.2 & 74.9 & 54.2 & \textbf{98.2} & 85.5 & 77.6 & \textbf{99.1} & \textbf{92.8} & 91.5 & 46.4 & 81.6 & 92.0 & \textbf{82.2} & \textbf{80.0} & \textbf{29.9} & 89.7 & 79.2 & 78.2 \\
 & LoRA32 & 92.7 & 70.4 & \textbf{84.6} & \textbf{79.8} & 88.2 & 84.7 & 71.2 & 59.9 & 95.6 & 83.3 & 71.9 & 96.7 & 87.5 & 90.6 & 49.0 & 95.2 & 90.4 & 76.6 & 70.2 & 18.4 & 91.8 & 79.8 & 78.6 \\
 & \method{}10 & 93.0 & 71.0 & 79.8 & 79.6 & \textbf{90.2} & 84.3 & \textbf{78.3} & 55.8 & 97.2 & \textbf{85.9} & \textbf{78.0} & 98.1 & 90.9 & \textbf{92.5} & \textbf{49.9} & 94.3 & \textbf{92.2} & 78.3 & 66.1 & 26.4 & \textbf{92.3} & 82.1 & \textbf{79.8} \\
 & FT & 92.2 & 69.9 & 81.7 & \textbf{79.8} & 88.2 & \textbf{85.3} & 76.2 & 56.2 & 95.8 & 83.3 & 73.3 & 97.6 & 89.1 & 91.7 & \textbf{49.8} & \textbf{95.6} & 90.6 & 76.9 & 66.1 & 24.6 & 91.9 & \textbf{82.4} & 79.0 \\
\midrule
\multirow{5}{*}{2 shots} & NoLA & 92.8 & 71.7 & 89.3 & 87.8 & 91.2 & 83.2 & 71.8 & 75.1 & 96.0 & 84.3 & 68.9 & 95.5 & 88.6 & 91.0 & 50.7 & 94.6 & 92.3 & 79.4 & 71.7 & 20.8 & 91.6 & 78.8 & 80.3 \\
 & VeRA1024 & \textbf{93.8} & 71.1 & 89.7 & 67.0 & 90.3 & 74.3 & \textbf{78.2} & 74.3 & \textbf{98.1} & 85.8 & \textbf{77.3} & \textbf{99.0} & \textbf{92.9} & 91.7 & 47.0 & 82.1 & 92.3 & \textbf{81.7} & \textbf{80.8} & \textbf{30.1} & 89.5 & 79.4 & 80.3 \\
 & LoRA32 & 93.1 & 71.9 & \textbf{93.2} & 84.9 & 92.2 & 86.0 & 72.3 & 77.4 & 96.9 & 83.5 & 70.4 & 94.6 & 87.7 & 91.9 & 51.9 & 97.2 & 91.8 & 77.5 & 75.5 & 20.8 & 92.9 & 83.6 & 81.2 \\
 & \method{}10 & \textbf{93.9} & \textbf{75.8} & 90.7 & 89.3 & 93.5 & \textbf{86.9} & \textbf{78.2} & \textbf{79.0} & 97.5 & \textbf{86.5} & 74.8 & 98.1 & 91.3 & \textbf{92.5} & \textbf{53.6} & \textbf{97.4} & \textbf{93.2} & 81.0 & 72.6 & 27.2 & \textbf{93.7} & 84.1 & \textbf{83.2} \\
 & FT & 93.1 & 74.0 & 90.8 & \textbf{89.9} & \textbf{93.7} & 86.1 & 74.3 & 74.0 & 95.4 & 85.2 & 71.2 & 97.2 & 90.1 & 92.0 & 46.9 & \textbf{97.4} & 92.6 & 78.2 & 71.5 & 25.1 & 93.3 & \textbf{84.5} & 81.7 \\
\midrule
\multirow{5}{*}{4 shots} & NoLA & 93.1 & 73.7 & 92.9 & 86.3 & 94.4 & 85.1 & 72.6 & 80.3 & 96.8 & 84.2 & 69.2 & 97.2 & 89.4 & 91.0 & 55.2 & 95.8 & 92.7 & 80.7 & 78.7 & 23.5 & 91.6 & 82.4 & 82.1 \\
 & VeRA1024 & 93.9 & 71.4 & 92.4 & 67.3 & 92.5 & 74.2 & \textbf{76.6} & 78.0 & \textbf{98.3} & 86.1 & 72.9 & \textbf{99.1} & \textbf{92.8} & 92.8 & 47.6 & 82.0 & \textbf{94.0} & 82.7 & \textbf{82.2} & \textbf{30.8} & 89.9 & 79.6 & 80.8 \\
 & LoRA32 & 93.9 & 73.7 & 94.2 & \textbf{89.5} & 95.6 & 87.8 & 72.5 & \textbf{80.9} & 97.1 & 85.3 & 70.8 & 97.3 & 89.1 & 92.3 & 56.9 & 97.8 & 92.4 & 82.4 & 78.5 & 23.3 & 91.9 & 84.8 & 83.1 \\
 & \method{}10 & \textbf{94.1} & \textbf{78.6} & \textbf{95.5} & \textbf{89.5} & \textbf{95.7} & \textbf{89.8} & \textbf{76.5} & 80.5 & 98.1 & \textbf{87.4} & \textbf{73.5} & 99.0 & 91.6 & 92.7 & \textbf{57.5} & 98.1 & 93.7 & \textbf{83.1} & 78.1 & 28.6 & \textbf{93.3} & \textbf{86.9} & \textbf{84.6} \\
 & FT & 93.8 & 78.1 & 94.0 & 88.5 & \textbf{95.8} & 89.4 & 75.5 & 77.1 & 96.2 & 86.5 & 72.8 & 97.7 & 90.5 & \textbf{93.5} & 51.8 & \textbf{98.4} & 92.8 & 81.5 & 77.2 & 25.2 & 93.0 & 86.3 & 83.4 \\
\midrule
\multirow{5}{*}{16 shots} & NoLA & 93.3 & 76.0 & 95.7 & 90.3 & 96.7 & 88.6 & 75.3 & 87.1 & 98.0 & 87.3 & 71.6 & 98.7 & 90.0 & 92.8 & 61.5 & 98.1 & 93.7 & 86.0 & 81.5 & 23.8 & 92.2 & 85.9 & 84.7 \\
 & VeRA1024 & 94.2 & 77.4 & 94.3 & 81.7 & 94.4 & 85.1 & 77.1 & 82.0 & \textbf{98.4} & 87.8 & 73.8 & \textbf{99.3} & \textbf{91.7} & 94.0 & 61.1 & 94.6 & \textbf{94.5} & 86.4 & 81.2 & \textbf{25.7} & 93.2 & 88.1 & 84.4 \\
 & LoRA32 & 93.5 & 77.7 & 95.3 & 92.5 & 96.6 & 90.2 & 75.7 & 86.8 & 98.2 & 88.3 & 73.2 & 98.5 & 90.4 & 93.9 & \textbf{65.5} & 98.8 & 92.9 & 86.8 & 80.9 & 23.2 & 92.2 & 87.7 & 85.4 \\
 & \method{}10 & \textbf{94.4} & 79.8 & \textbf{95.9} & 92.3 & \textbf{96.9} & \textbf{91.7} & \textbf{78.1} & \textbf{87.4} & 98.1 & \textbf{88.6} & \textbf{75.6} & 99.0 & 91.2 & \textbf{94.6} & 64.5 & \textbf{99.0} & 94.0 & \textbf{87.8} & 80.7 & 25.2 & 92.2 & 89.1 & \textbf{86.2} \\
 & FT & 94.3 & \textbf{80.0} & \textbf{95.8} & \textbf{94.1} & 96.5 & \textbf{91.7} & 77.5 & 85.6 & 97.8 & 87.6 & 75.0 & 98.6 & 90.8 & 94.5 & 64.7 & 98.7 & 93.7 & 87.1 & \textbf{81.9} & 25.2 & \textbf{93.3} & \textbf{89.6} & 86.1 \\
\midrule
\multirow{5}{*}{50\%} & NoLA & 93.0 & 80.8 & \textbf{99.1} & \textbf{96.7} & \textbf{99.2} & 95.2 & 75.2 & 96.2 & 99.2 & 90.7 & 74.7 & 99.3 & 93.0 & 95.3 & 70.5 & 97.5 & 94.5 & 85.2 & 90.9 & 25.9 & 91.7 & 87.4 & 87.8 \\
 & VeRA1024 & 93.8 & 82.0 & \textbf{99.2} & 96.2 & \textbf{99.3} & 96.1 & 76.9 & 96.2 & 99.2 & 91.9 & 76.3 & \textbf{99.5} & 93.6 & 96.2 & 72.7 & 98.3 & \textbf{95.2} & 86.7 & 90.5 & 25.7 & 95.8 & 89.4 & 88.7 \\
 & LoRA32 & 92.7 & 82.1 & 98.8 & 96.5 & \textbf{99.3} & 96.2 & 75.7 & 96.5 & \textbf{99.3} & 91.9 & 77.0 & 99.4 & 94.2 & 96.0 & 74.0 & 97.2 & 94.5 & 86.4 & 89.5 & 25.9 & \textbf{96.2} & 90.5 & 88.6 \\
 & \method{}10 & \textbf{94.8} & \textbf{82.8} & 98.8 & 96.6 & \textbf{99.3} & 96.4 & \textbf{77.7} & \textbf{96.8} & \textbf{99.3} & \textbf{93.0} & \textbf{79.1} & \textbf{99.5} & \textbf{94.6} & \textbf{96.5} & 77.2 & \textbf{98.7} & 94.7 & \textbf{87.3} & \textbf{91.3} & \textbf{28.1} & 96.0 & 90.4 & \textbf{89.5} \\
 & FT & 94.5 & 82.0 & 98.8 & \textbf{96.8} & \textbf{99.3} & \textbf{96.6} & 76.7 & \textbf{96.8} & 99.0 & 91.8 & 77.0 & 99.4 & 94.2 & \textbf{96.4} & \textbf{77.9} & 98.1 & 94.6 & 86.7 & 90.8 & 27.6 & 95.5 & \textbf{91.4} & 89.2 \\
\midrule
\multirow{5}{*}{100\%} & NoLA & 93.3 & 84.2 & \textbf{99.3} & 96.7 & 99.4 & 96.2 & 76.6 & 96.8 & 99.2 & 91.5 & 74.8 & 99.5 & 93.7 & 95.5 & 77.2 & 98.7 & 94.4 & 86.9 & 91.4 & 28.2 & 95.0 & 89.8 & 89.0 \\
 & VeRA1024 & 94.3 & 85.0 & 99.0 & 97.2 & 99.4 & 97.0 & 78.0 & 96.8 & \textbf{99.3} & 92.6 & 76.7 & \textbf{99.6} & 94.3 & 95.9 & 79.7 & \textbf{99.3} & 95.0 & 87.7 & 91.3 & 27.3 & 96.4 & 91.3 & 89.7 \\
 & LoRA32 & 93.1 & 85.8 & 99.1 & 97.3 & \textbf{99.5} & 97.2 & 77.6 & 97.2 & \textbf{99.3} & 93.0 & 77.9 & \textbf{99.5} & 94.9 & 96.6 & 83.7 & 99.0 & 94.4 & 87.5 & \textbf{91.8} & 28.3 & 95.9 & 91.5 & 90.0 \\
 & \method{}10 & 94.7 & \textbf{86.0} & 99.0 & 97.0 & \textbf{99.4} & 97.1 & \textbf{79.4} & \textbf{97.3} & \textbf{99.3} & \textbf{93.5} & \textbf{80.1} & 99.5 & \textbf{95.2} & \textbf{97.1} & 84.1 & \textbf{99.3} & 95.1 & \textbf{88.6} & 91.7 & \textbf{31.2} & \textbf{96.5} & 92.7 & \textbf{90.6} \\
 & FT & \textbf{94.9} & 84.2 & 98.8 & \textbf{97.5} & \textbf{99.5} & \textbf{97.6} & 78.7 & \textbf{97.3} & 99.2 & 92.8 & 77.9 & 99.5 & 94.8 & 96.8 & \textbf{84.4} & \textbf{99.3} & \textbf{95.3} & 88.3 & \textbf{91.8} & 30.1 & \textbf{96.6} & \textbf{93.0} & 90.4 \\
    \bottomrule
    \end{tabularx}
\end{table}

\subsection{Vision only: DinoV2\label{app:dinoresults}}
Table~\ref{tab:resultsdetaildinov2} reports detailed results when fine-tuning DinoV2 on 21 datasets. We use the pre-trained ViT-B/14 architecture and train a linear classifier together with the feature extractor. Compared to the CLIP results ImageNet was removed to promote brevity of the experiments.
\begin{table}[t]
\captionsetup{font=footnotesize}
\caption{Detailed accuracy results per dataset, the DinoV2 ViT-B/14 vision backbone.
Highest performance and those within a range of $0.1$ in each section are highlighted in bold.\label{tab:resultsdetaildinov2}}
    \centering
    \setlength{\tabcolsep}{1pt} 
    \scriptsize
    \begin{tabularx}{\linewidth}{@{\extracolsep{\fill}} ll | ccccccccccccccccccccc c}
    \toprule
    & Method & \rotatebox[origin=c]{90}{Cars} & \rotatebox[origin=c]{90}{DTD} & \rotatebox[origin=c]{90}{EuroSAT} & \rotatebox[origin=c]{90}{GTSRB} & \rotatebox[origin=c]{90}{MNIST} & \rotatebox[origin=c]{90}{RESISC45} & \rotatebox[origin=c]{90}{SUN397} & \rotatebox[origin=c]{90}{SVHN} & \rotatebox[origin=c]{90}{CIFAR10} & \rotatebox[origin=c]{90}{CIFAR100} & \rotatebox[origin=c]{90}{STL10} & \rotatebox[origin=c]{90}{Food101} & \rotatebox[origin=c]{90}{Caltech256} & \rotatebox[origin=c]{90}{FGVCAircraft} & \rotatebox[origin=c]{90}{Flowers102} & \rotatebox[origin=c]{90}{OxfordIIITPet} & \rotatebox[origin=c]{90}{CUB200} & \rotatebox[origin=c]{90}{PascalVOC} & \rotatebox[origin=c]{90}{Country211} & \rotatebox[origin=c]{90}{Caltech101} & \rotatebox[origin=c]{90}{UCF101} & \rotatebox[origin=c]{90}{Average}\tabularnewline
    \midrule
    \multirow{5}{*}{1 shots} &  NoLA & 21.2 & 45.4 & 60.7 & 28.8 & 55.0 & \textbf{49.8} & 46.3 & 14.4 & 73.8 & \textbf{57.3} & 71.7 & 50.5 & \textbf{78.7} & \textbf{19.7} & 98.6 & 74.6 & 62.5 & \textbf{43.6} & 3.1 & 85.5 & \textbf{63.8} & \textbf{52.6} \\
 &  VeRA256 & \textbf{22.5} & 45.6 & 57.9 & 20.1 & 50.7 & 44.6 & \textbf{46.7} & 12.9 & 76.5 & 55.8 & 64.4 & 51.9 & 78.6 & 19.1 & 98.7 & 75.5 & 62.9 & 36.2 & \textbf{3.4} & 84.7 & 63.1 & 51.0 \\
 &  LoRA32 & \textbf{22.6} & 47.2 & 59.3 & 24.8 & 51.7 & 48.7 & 45.9 & 14.6 & \textbf{77.2} & \textbf{57.4} & 64.6 & \textbf{52.4} & 77.5 & \textbf{19.7} & \textbf{98.9} & \textbf{76.5} & \textbf{63.1} & 37.2 & \textbf{3.4} & 85.1 & 62.3 & 51.9 \\
 &  RandLoRA6 & 21.5 & \textbf{47.8} & 57.9 & \textbf{34.5} & \textbf{61.6} & 44.9 & 44.1 & \textbf{16.2} & 56.8 & 54.0 & 66.0 & 47.2 & 76.4 & \textbf{19.6} & 97.8 & 71.8 & 59.9 & 43.4 & 3.0 & \textbf{86.1} & 62.8 & 51.1 \\
 &  FT & 20.8 & 45.5 & \textbf{67.8} & 25.7 & 52.3 & 45.2 & 45.3 & 15.5 & 70.2 & 54.7 & \textbf{75.8} & 50.1 & 75.2 & 19.4 & 98.3 & 70.8 & 60.0 & 36.9 & 3.1 & 84.7 & 60.3 & 51.3 \\
\midrule
\multirow{5}{*}{2 shots} &  NoLA & 41.4 & 57.8 & 64.1 & 43.1 & \textbf{73.5} & \textbf{65.4} & 58.1 & 16.2 & \textbf{90.0} & 74.5 & \textbf{93.9} & 64.5 & 84.9 & 28.0 & \textbf{99.6} & 83.0 & 73.3 & 51.4 & 4.1 & \textbf{90.0} & 72.7 & \textbf{63.3} \\
 &  VeRA256 & 38.2 & 57.4 & 64.0 & 28.9 & 65.0 & 60.8 & 57.3 & 14.4 & 86.0 & 71.6 & 78.9 & \textbf{66.2} & 84.8 & 26.2 & 99.4 & 83.4 & \textbf{74.8} & 44.4 & 4.1 & 88.0 & \textbf{73.7} & 60.4 \\
 &  LoRA32 & 41.1 & 58.4 & 68.3 & 37.7 & 71.4 & 64.7 & \textbf{58.2} & 14.7 & 89.6 & \textbf{74.8} & 87.0 & 66.1 & \textbf{85.3} & 27.0 & \textbf{99.5} & \textbf{86.7} & 73.0 & \textbf{52.1} & \textbf{5.0} & 89.2 & 72.1 & 63.0 \\
 &  RandLoRA6 & 41.9 & \textbf{59.6} & \textbf{69.0} & \textbf{48.6} & 70.2 & 62.2 & 57.1 & \textbf{19.4} & 72.1 & 70.6 & 84.8 & 63.3 & 83.8 & \textbf{29.5} & 98.4 & 80.0 & 71.3 & 49.7 & 3.8 & \textbf{89.9} & 72.1 & 61.8 \\
 &  FT & \textbf{43.1} & 56.0 & 65.7 & 42.6 & 72.1 & 63.1 & 57.5 & 16.1 & 79.0 & 71.0 & 91.3 & 64.6 & 85.1 & 27.7 & 99.3 & 81.1 & 71.8 & 50.6 & 3.9 & 89.2 & 70.5 & 62.0 \\
\midrule
\multirow{5}{*}{4 shots} &  NoLA & 62.9 & \textbf{68.5} & 76.4 & 62.4 & 82.4 & 75.9 & 65.8 & 22.8 & \textbf{94.5} & \textbf{82.6} & \textbf{97.6} & 73.7 & 89.8 & 40.6 & \textbf{99.7} & 89.3 & 82.6 & 65.4 & 6.0 & 90.6 & 79.5 & \textbf{71.9} \\
 &  VeRA256 & 56.1 & 64.2 & 71.5 & 43.2 & 76.1 & 71.6 & 64.8 & 17.6 & 91.4 & 80.9 & 88.7 & 74.7 & 89.3 & 36.0 & \textbf{99.7} & 91.0 & 82.1 & 53.5 & 6.2 & 89.8 & 77.7 & 67.9 \\
 &  LoRA32 & 63.4 & 66.5 & \textbf{79.2} & 61.0 & 79.2 & \textbf{77.6} & \textbf{66.3} & 20.9 & \textbf{94.5} & 82.1 & 94.5 & \textbf{75.0} & 89.4 & 41.9 & \textbf{99.7} & \textbf{91.9} & \textbf{83.5} & \textbf{66.2} & \textbf{6.8} & 89.7 & \textbf{80.6} & \textbf{71.9} \\
 &  RandLoRA6 & 64.6 & 65.3 & 72.2 & \textbf{66.6} & \textbf{86.4} & 77.0 & 65.0 & \textbf{24.8} & 84.0 & 79.4 & 93.1 & 73.0 & 89.8 & \textbf{43.9} & 99.6 & 86.6 & 82.4 & 63.8 & 5.9 & \textbf{91.7} & 78.6 & 71.1 \\
 &  FT & \textbf{65.5} & 67.3 & 73.0 & 62.7 & 85.6 & 73.8 & 66.0 & 20.9 & 88.0 & 81.3 & 94.0 & 73.6 & \textbf{90.2} & 41.4 & \textbf{99.6} & 88.9 & 82.6 & 61.3 & 6.0 & \textbf{91.7} & 79.3 & 71.1 \\
\midrule
\multirow{5}{*}{16 shots} &  NoLA & 86.4 & 79.5 & 91.0 & 88.2 & 94.0 & 87.8 & \textbf{75.1} & 56.9 & 97.2 & \textbf{89.4} & 99.0 & 85.1 & 93.1 & 64.0 & \textbf{99.7} & 93.9 & \textbf{88.8} & \textbf{82.2} & 12.2 & \textbf{95.3} & 87.0 & 83.1 \\
 &  VeRA256 & 81.7 & 78.2 & 88.2 & 60.1 & 88.9 & 83.2 & 73.5 & 30.2 & \textbf{97.4} & 87.6 & 97.4 & 84.4 & 92.6 & 51.3 & \textbf{99.7} & \textbf{94.6} & 88.5 & 72.9 & 11.8 & 91.9 & 85.7 & 78.1 \\
 &  LoRA32 & 87.1 & \textbf{80.5} & \textbf{93.9} & 86.4 & 93.0 & 87.2 & \textbf{75.1} & 44.8 & \textbf{97.4} & 88.8 & \textbf{99.4} & \textbf{85.6} & \textbf{93.5} & 65.2 & \textbf{99.7} & 94.2 & 88.5 & 80.5 & \textbf{12.4} & 94.6 & 87.6 & 82.6 \\
 &  RandLoRA6 & \textbf{88.4} & 79.0 & 92.3 & \textbf{90.3} & \textbf{95.4} & 87.3 & 74.7 & \textbf{57.4} & 97.0 & 88.5 & 98.2 & 85.5 & 93.1 & \textbf{71.5} & \textbf{99.7} & 93.3 & 88.6 & 79.7 & 11.8 & 94.5 & \textbf{87.8} & \textbf{83.5} \\
 &  FT & 87.3 & 78.8 & 92.4 & 88.9 & 95.0 & \textbf{88.9} & 74.7 & 50.4 & 96.8 & 88.6 & 98.8 & 85.3 & 93.4 & 67.1 & \textbf{99.7} & 93.2 & \textbf{88.9} & 77.8 & 11.8 & 94.9 & 87.4 & 82.9 \\
\midrule
\multirow{5}{*}{0.5 shots} &  NoLA & 89.0 & 82.5 & 98.9 & 96.4 & 99.2 & 94.8 & 76.0 & 96.5 & \textbf{99.2} & 93.2 & \textbf{99.6} & 92.5 & \textbf{95.8} & 73.7 & 99.5 & 94.9 & 87.4 & 92.8 & 18.7 & \textbf{97.8} & 88.6 & 88.9 \\
 &  VeRA256 & 84.0 & 80.1 & 97.3 & 89.7 & 97.7 & 92.0 & 74.8 & 88.2 & 99.0 & 92.2 & 99.4 & 91.7 & 94.7 & 68.4 & 99.5 & \textbf{95.1} & 86.9 & 89.9 & 17.6 & 96.0 & 87.5 & 86.7 \\
 &  LoRA32 & 89.7 & 82.8 & \textbf{99.0} & 96.2 & 99.1 & 94.8 & 75.9 & 96.6 & \textbf{99.3} & \textbf{93.7} & \textbf{99.5} & \textbf{93.2} & 95.5 & 72.6 & 98.3 & \textbf{95.0} & \textbf{88.6} & \textbf{93.0} & 19.0 & 97.5 & 90.3 & 89.0 \\
 &  RandLoRA6 & 89.7 & \textbf{83.2} & 98.7 & \textbf{97.1} & \textbf{99.3} & 95.5 & 75.8 & \textbf{97.2} & \textbf{99.2} & 93.4 & \textbf{99.6} & \textbf{93.3} & 95.5 & 75.2 & \textbf{99.7} & 94.9 & 87.5 & 92.8 & 19.6 & 97.6 & 89.3 & \textbf{89.2} \\
 &  FT & \textbf{90.3} & 81.5 & 98.8 & 96.6 & \textbf{99.3} & \textbf{95.8} & \textbf{76.2} & 96.6 & \textbf{99.2} & 93.4 & 99.3 & 93.0 & 95.7 & \textbf{75.3} & 98.9 & 95.0 & 87.4 & 92.4 & \textbf{19.9} & 97.3 & \textbf{90.6} & \textbf{89.2} \\
\midrule
\multirow{5}{*}{1.0 shots} &  NoLA & 92.5 & 85.4 & 98.8 & 96.9 & \textbf{99.3} & 96.1 & 77.8 & 96.8 & \textbf{99.4} & 94.1 & \textbf{99.7} & 93.4 & \textbf{96.2} & 81.8 & \textbf{99.7} & \textbf{95.9} & 90.2 & \textbf{93.6} & 22.3 & \textbf{98.2} & 90.2 & 90.4 \\
 &  VeRA256 & 89.8 & 81.6 & 97.4 & 89.5 & 98.1 & 93.1 & 76.5 & 88.4 & 99.1 & 92.6 & \textbf{99.6} & 92.5 & 95.3 & 75.4 & \textbf{99.7} & 95.8 & 89.7 & 90.6 & 20.5 & 97.1 & 88.2 & 88.1 \\
 &  LoRA32 & 92.7 & 84.6 & 99.1 & 96.3 & \textbf{99.3} & 96.0 & 78.2 & 97.2 & 99.3 & 94.2 & \textbf{99.7} & 93.7 & \textbf{96.3} & 83.3 & \textbf{99.7} & 95.7 & \textbf{90.4} & 92.7 & 20.6 & 97.8 & 91.5 & 90.4 \\
 &  RandLoRA6 & \textbf{93.3} & \textbf{85.5} & 99.0 & \textbf{97.1} & \textbf{99.4} & 96.8 & 77.9 & \textbf{97.5} & \textbf{99.5} & \textbf{94.4} & \textbf{99.7} & \textbf{94.2} & \textbf{96.2} & \textbf{84.0} & \textbf{99.6} & 95.8 & 90.1 & 93.1 & 22.7 & 98.0 & \textbf{92.0} & \textbf{90.8} \\
 &  FT & \textbf{93.4} & \textbf{85.5} & \textbf{99.4} & 96.8 & \textbf{99.3} & \textbf{97.0} & \textbf{78.4} & 97.4 & 99.2 & 94.0 & \textbf{99.6} & 94.1 & \textbf{96.3} & \textbf{83.9} & \textbf{99.7} & 95.8 & 90.1 & 93.1 & \textbf{23.8} & 98.0 & 91.8 & \textbf{90.8} \\
    \bottomrule
    \end{tabularx}
\end{table}

\subsection{Commonsense reasoning}
Table~\ref{tab:comsense} reports detailed accuracy results for the Qwen2, Phi3 and LLama3 language models trained on the commonsense tasks.
See~\ref{app:commonsensedatasets} for details on the datasets and the hyper-parameters used.

\begin{table}[ht]
\centering
\scalebox{0.75}{
\begin{tabular}{lcccccccccl}
\toprule
Method & \% Params & BoolQ & PIQA & SIQA & HellaSwag & WinoGrande & ARC-e & ARC-c & OBQA & Average + $\Delta$ \\
\midrule
\multicolumn{11}{c}{Qwen2 - Zero-shot} \\
\midrule
Zero-shot & 0  & 3.12 & 4.68 & 7.22 & 2.50 & 14.52 & 4.80 & 1.79 & 2.60 & 5.15 \\
\midrule
\multicolumn{11}{c}{Qwen2 - 15k} \\
\midrule
NoLA & 0.05 & 54.16 & 56.91 & 47.65 & 17.36 & 45.46 & 46.55 & 32.51 & 39.80 & 42.55 \\
VeRA1024 & 0.06  & 58.78 & 56.64 & 50.10 & 24.95 & 49.80 & 56.52 & 37.80 & 50.40 & 48.12 \\
LoRA-16 & 1.18   & 62.14 & 62.13 & 58.24 & 27.86 & 49.96 & 62.46 & 44.97 & 58.20 & 53.25 \\
\method{}-10 & 1.18  & 62.14 & 63.49 & 55.32 & 31.16 & 49.96 & 64.27 & 44.97 & 56.60 & 53.49 \textcolor{color3}{+0.24}  \\
LoRA-32 & 2.33 & 59.94 & 62.13 & 56.55 & 30.27 & 41.99 & 64.39 & 46.42 & 57.00 & 52.34 \\
\method{}-5 & 2.33  & 62.81 & 63.82 & 54.86 & 30.00 & 48.07 & 64.81 & 43.34 & 55.40 & 52.89 \textcolor{color3}{+0.55} \\
\midrule
\multicolumn{11}{c}{Qwen2 - 170k} \\
\midrule
NoLA & 0.05  & 55.99 & 52.50 & 55.07 & 23.74 & 50.51 & 55.64 & 38.91 & 46.80 & 47.40 \\
VeRA1024 & 0.06 & 55.50 & 59.30 & 52.81 & 34.52 & 52.72 & 58.55 & 42.94 & 57.80 & 51.78 \\
LoRA-16 & 1.18 & 53.39 & 68.12 & 66.33 & 46.46 & 58.72 & 59.97 & 43.77 & 62.20 & 57.37 \\
\method{}-10 & 1.18  & 61.47 & 67.63 & 65.61 & 40.26 & 57.22 & 62.12 & 47.95 & 59.60 & 57.73 \textcolor{color3}{+0.36} \\
LoRA-32 & 2.33  & 55.78 & 68.28 & 67.20 & 42.37 & 60.22 & 61.03 & 45.05 & 58.80 & 57.34 \\
\method{}-5 & 2.33  & 63.46 & 65.72 & 66.43 & 42.90 & 56.20 & 61.49 & 47.53 & 59.20 & 57.86 \textcolor{color3}{+0.52} \\
\midrule
\multicolumn{11}{c}{Phi3 - Zero-shot} \\
\midrule
Zero-shot & 0 & 62.26 & 79.82 & 65.81 & 56.29 & 19.89 & 89.86 & 77.65 & 71.40 & 65.37 \\
\midrule
\multicolumn{11}{c}{Phi3 - 15k} \\
\midrule
NoLA & 0.005 & 66.24 & 85.15 & 73.49 & 78.29 & 73.95 & 95.33 & 85.15 & 85.20 & 80.35 \\
VeRA1024 & 0.015 & 68.53 & 84.49 & 73.08 & 74.54 & 72.85 & 93.01 & 80.97 & 81.60 & 78.63 \\
LoRA-16 & 0.57  & 69.51 & 85.36 & 75.44 & 80.15 & 75.85 & 95.37 & 86.09 & 86.60 & 81.80 \\
\method{}-40 & 0.58 & 69.54 & 85.31 & 73.80 & 84.05 & 75.14 & 94.65 & 84.90 & 85.80 & 81.65 \textcolor{red}{-0.15} \\
LoRA-32 & 1.14 & 68.44 & 85.31 & 74.67 & 72.14 & 74.98 & 95.20 & 85.41 & 86.60 & 80.34 \\
\method{}-20 &  1.16 & 69.20 & 85.42 & 75.33 & 83.98 & 75.77 & 95.50 & 85.92 & 87.60 & 82.33 \textcolor{color3}{+1.99}\\
LoRA-64 & 2.28  & 69.88 & 85.75 & 74.97 & 74.45 & 75.30 & 95.54 & 87.12 & 88.00 & 81.37 \\
\method{}-10 & 2.29 & 69.63 & 85.31 & 75.03 & 86.94 & 75.30 & 95.24 & 85.58 & 86.40 & 82.43 \textcolor{color3}{+1.06}\\
\midrule
\multicolumn{11}{c}{Phi3 - 170k} \\
\midrule
NoLA & 0.005 & 68.87 & 85.15 & 77.18 & 85.13 & 77.90 & 95.20 & 85.58 & 83.60 & 82.33 \\
VeRA1024 & 0.015 & 69.53 & 84.53 & 74.52 & 84.08 & 76.82 & 94.51 & 83.68 & 83.54 & 81.40 \\
LoRA-16 & 0.57  & 70.83 & 84.39 & 78.45 & 89.94 & 82.87 & 95.45 & 86.09 & 89.00 & 84.63 \\
\method{}-40 & 0.58 & 70.86 & 86.67 & 78.81 & 90.07 & 82.00 & 95.12 & 86.26 & 87.60 & 84.67 \textcolor{color3}{+0.04}\\
LoRA-32 & 1.14  & 71.23 & 85.96 & 78.92 & 91.77 & 82.95 & 94.61 & 84.81 & 89.40 & 84.96 \\
\method{}-20 & 1.16 & 71.62 & 87.43 & 79.48 & 91.48 & 82.79 & 95.16 & 86.01 & 87.80 & 85.22 \textcolor{color3}{+0.26}\\
LoRA-64 & 2.28  & 71.93 & 86.13 & 79.58 & 90.14 & 83.74 & 92.68 & 81.74 & 87.80 & 84.22 \\
\method{}-10 & 2.29 & 71.87 & 86.56 & 79.43 & 90.99 & 82.72 & 95.66 & 85.49 & 87.40 & 85.01 \textcolor{color3}{+0.79}\\

\midrule
\multicolumn{11}{c}{LLama3 - Zero-shot} \\
\midrule
Zero-shot & 0 & 60.73 & 41.40 & 28.40 & 25.00 & 10.97 & 16.41 & 15.96 & 16.80 & 26.96 \\
\midrule
\multicolumn{11}{c}{LLama3 - 15k} \\
\midrule
NoLA &  0.004 & 67.58 & 84.49 & 72.31 & 69.60 & 70.56 & 90.49 & 78.75 & 81.20 & 76.87 \\
VeRA1024 &  0.014 & 63.36 & 84.39 & 74.10 & 77.70 & 71.35 & 89.48 & 76.54 & 80.20 & 77.14 \\
LoRA-16 & 0.35  & 73.03 & 86.94 & 75.90 & 90.53 & 77.74 & 90.74 & 80.29 & 86.20 & 82.67 \\
\method{}-60 & 0.36  & 71.19 & 84.22 & 75.59 & 83.82 & 74.98 & 91.12 & 81.31 & 86.00 & 81.03 \textcolor{red}{-1.64}\\
LoRA-32 & 0.7  & 74.22 & 86.40 & 75.79 & 91.90 & 77.35 & 90.61 & 80.80 & 87.60 & 83.09 \\
\method{}-30 & 0.7  & 71.65 & 83.79 & 74.56 & 86.85 & 75.61 & 90.78 & 80.03 & 87.20 & 81.31 \textcolor{red}{-1.78} \\ 
LoRA-64 & 1.4  & 71.77 & 84.17 & 76.25 & 85.14 & 73.80 & 91.46 & 80.80 & 86.20 & 81.20 \\
\method{}-15 & 1.4 & 70.98 & 86.02 & 75.44 & 89.74 & 76.80 & 91.29 & 81.66 & 83.80 & 81.96 \textcolor{color3}{+0.76}\\
\midrule
\multicolumn{11}{c}{LLama3 - 170k} \\
\midrule
NoLA & 0.004 & 71.83 & 84.66 & 77.79 & 85.05 & 82.72 & 88.59 & 76.45 & 82.20 & 81.16 \\
VeRA1024 & 0.014 & 70.55 & 85.69 & 79.27 & 92.14 & 82.64 & 87.33 & 73.38 & 82.20 & 81.65 \\
LoRA-16 & 0.35  & 75.14 & 89.12 & 80.66 & 89.01 & 86.58 & 90.07 & 78.75 & 86.20 & 84.44 \\
\method{}-60 & 0.35 & 75.26 & 87.98 & 79.63 & 94.66 & 85.64 & 90.03 & 79.44 & 84.40 & 84.62 \textcolor{color3}{+0.18}\\
LoRA-32  & 0.7   & 75.08 & 88.85 & 80.25 & 95.42 & 86.19 & 90.28 & 80.29 & 85.60 & 85.24 \\
\method{}-30 & 0.7 & 76.33 & 88.08 & 80.25 & 95.67 & 86.11 & 90.36 & 80.89 & 87.00 & 85.59 \textcolor{color3}{+0.45} \\
LoRA-64  & 1.4 & 74.65 & 89.66 & 80.86 & 95.17 & 86.74 & 90.95 & 79.18 & 85.40 & 85.33 \\
\method{}-15 & 1.4 & 72.63 & 87.98 & 81.37 & 95.68 & 87.77 & 91.33 & 80.89 & 89.00 & 85.83 \textcolor{color3}{+0.50} \\

\bottomrule
\end{tabular}}
\caption{Comparison of accuracy on commonsense reasoning datasets. We report accuracy delta of \method{} with LoRA for comparable amounts of trainable parameters.}
\label{tab:comsense}
\end{table}

\end{document}